\documentclass{article}

\PassOptionsToPackage{numbers, sort&compress}{natbib}

\usepackage[final]{neurips_2024}

\usepackage[utf8]{inputenc} %
\usepackage[T1]{fontenc}    %
\usepackage{hyperref}       %
\usepackage{url}            %
\usepackage{booktabs}       %
\usepackage{nicefrac}       %
\usepackage{microtype}      %
\usepackage{colorlib}         %

\definecolor{link color}{HTML}{626572}
\hypersetup{
    colorlinks=true,
    citecolor=link color,
    linkcolor=link color,
    urlcolor=link color,
}

\usepackage{float}
\usepackage{subcaption}
\usepackage{enumitem}

\usepackage{amsmath}
\usepackage{amssymb}
\usepackage{mathtools}
\usepackage{amsthm}
\usepackage{amsfonts}
\usepackage{wrapfig}
\usepackage{pifont}
\usepackage{pgfplots}
\pgfplotsset{compat=1.18}
\usepackage{amsmath,amsfonts,bm,cancel}

\newcommand{\figleft}{{\em (Left)}}
\newcommand{\figtopleft}{{\em (Top Left)}}
\newcommand{\figtopright}{{\em (Top Right)}}
\newcommand{\figbottomleft}{{\em (Bottom Left)}}
\newcommand{\figbottomright}{{\em (Bottom Right)}}

\newcommand{\figcenter}{{\em (Center)}}
\newcommand{\figright}{{\em (Right)}}

\newcommand{\figbottom}{{\em (Bottom)}}

\newcommand{\dd}{\mathop{}\!\mathrm{d}}
\newcommand{\cmark}{\ding{51}}%
\newcommand{\xmark}{\ding{55}}%
\def\dx{\dd x}

\def\eqref#1{equation~\ref{#1}}

\def\1{\bm{1}}

\DeclareMathAlphabet{\mathsfit}{\encodingdefault}{\sfdefault}{m}{sl}
\SetMathAlphabet{\mathsfit}{bold}{\encodingdefault}{\sfdefault}{bx}{n}

\def\gH{{\mathcal{H}}}

\def\gL{{\mathcal{L}}}

\def\gN{{\mathcal{N}}}

\newcommand{\pdata}{p_{\rm{data}}}

\DeclareMathOperator{\E}{\mathbb{E}}

\DeclareMathOperator*{\argmax}{arg\,max}
\DeclareMathOperator*{\argmin}{arg\,min}

\usepackage{preamble}

\usepackage[capitalize,nameinlink]{cleveref}
\usepackage{titlesec}

\titlespacing*{\section}       {0pt}{1.3ex plus 0.6ex minus 0.4ex}{0.8ex plus 0.2ex minus 0.3ex}
\titlespacing*{\subsection}    {0pt}{1.1ex plus 0.4ex minus 0.3ex}{0.6ex plus 0.2ex minus 0.3ex}
\titlespacing*{\subsubsection} {0pt}{0.9ex plus 0.2ex minus 0.2ex}{0.4ex plus 0.2ex minus 0.3ex}
\titlespacing*{\paragraph}     {0pt}{0.5ex plus 0.5ex}{1ex}
\newcommand{\URL}{\url{https://github.com/vivekmyers/contrastive_planning}}

\DeclareRobustCommand{\bbm}[1]{\operatorname{\text{\usefont{U}{DSSerif}{m}{n}#1}}}
\def\1{\bbm{1}}

\title{Inference via Interpolation: Contrastive~Representations Provably Enable Planning and Inference}

\makeatletter
\renewcommand{\@maketitle}{%
    \vbox{%
        \hsize\textwidth%
        \linewidth\hsize%
        \vskip 0.1in%
        \@toptitlebar%
        \centering%
        {\LARGE\bf \@title\par}%
        \@bottomtitlebar%
        \def\And{%
            \end{tabular}\hss\egroup%
            \hspace*{\the\glueexpr\fill*3/2}%
            \hbox to 0pt\bgroup%
            \hss\begin{tabular}[t]{c}\bf\rule{\z@}{24\p@}\ignorespaces%
        }%
        \def\AND{%
            \end{tabular}\hss\egroup\hspace*{\fill}\linebreak\hspace*{\fill}%
            \hbox to 0pt\bgroup\hss%
            \begin{tabular}[t]{c}\bf\rule{\z@}{24\p@}\ignorespaces%
        }%
        \parbox{\textwidth}{%
            \hspace*{\fill}%
            \hbox to 0pt{%
                \hss\begin{tabular}[t]{c}\bf\rule{\z@}{24\p@}\@author\end{tabular}\hss%
            }%
            \hspace*{\fill}%
        }%
        \vskip 0.3in \@minus 0.1in%
    }%
}
\makeatother

\author{%
    Benjamin Eysenbach \equal \\%
    Princeton University\\%
    \texttt{eysenbach@princeton.edu} \\%
    \And%
    Vivek Myers \equal \\%
    UC Berkeley \\%
    \texttt{vmyers@berkeley.edu} \\%
    \AND%
    Ruslan Salakhutdinov \\%
    Carnegie Mellon University \\%
    \texttt{rsalakhu@cs.cmu.edu} \\%
    \And%
    Sergey Levine \\%
    UC Berkeley \\%
    \texttt{svlevine@eecs.berkeley.edu} \\%
}

\begin{document}

\maketitle

\begin{abstract}
    Given time series data, how can we answer questions like ``what will happen in the future?'' and ``how did we get here?''
    These sorts of probabilistic inference questions are challenging when observations are high-dimensional.
    In this paper, we show how these questions can have compact, closed form solutions in terms of learned representations. The key idea is to apply a variant of contrastive learning to time series data. Prior work already shows that the representations learned by contrastive learning encode a probability ratio. By extending prior work to show that the marginal distribution over representations is Gaussian, we can then prove that joint distribution of representations is also Gaussian. Taken together, these results show that representations learned via temporal contrastive learning follow a Gauss-Markov chain, a graphical model where inference (e.g., prediction, planning) over representations corresponds to inverting a low-dimensional matrix. In one special case, inferring intermediate representations will be equivalent to interpolating between the learned representations. We validate our theory using numerical simulations on tasks up to 46-dimensions.\footnote{Code: \URL}
\end{abstract}

\section{Introduction}

Probabilistic modeling of time-series data has applications ranging from robotic control~\citep{theodorou2010generalized} to material science~\citep{jonsson1998nudged}, from cell biology~\citep{saelens2019comparison} to astrophysics~\citep{majewski2017apache}.
These applications are often concerned with two questions: \emph{predicting} future states (e.g., what will this cell look like in an hour), and \emph{inferring} trajectories between two given states.
However, answering these questions often requires reasoning over high-dimensional data, which can be challenging as most tools in the standard probabilistic toolkit require generation. Might it be possible to use discriminative methods (e.g., contrastive learning) to perform such inferences?

\begin{figure}
    \ignorespaces
    \centering
    \includegraphics{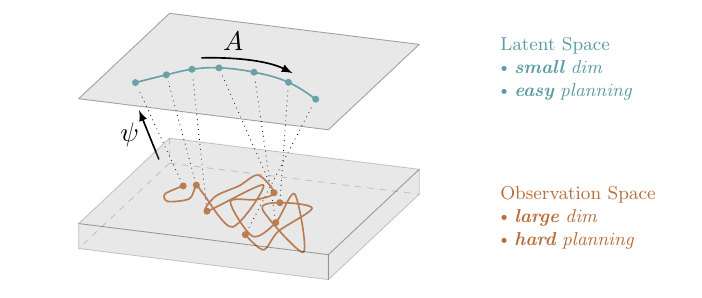}

    \caption{
        We apply temporal contrastive learning to observation pairs to obtain representations ($\psi(x_0), \psi(x_{t+k})$) such that $A \psi(x_0)$ is close to $\psi(x_{t+k})$.
        While inferring waypoints in the high-dimensional observation space is challenging, we show that the distribution over intermediate latent representations has a closed form solution corresponding to linear interpolation between the initial and final representations.
    }
    \label{fig:intuition-as-model}
\end{figure}

Many prior works aim to learn representations that are easy to predict while retaining salient bits of information. %
For time-series data, we want the representation to remain a sufficient statistic for distributions related to time \-- for example, they should retain bits required to predict future states (or representations thereof).
While generative methods~\citep{zhao2017towards, zhu2020s3vae, makhzani2015adversarial, dumoulin2016adversarially} have this property,
they tend to be computationally expensive (see, e.g.,~\citep{razavi2019generating}) and can be challenging to scale to high-dimensional observations.

In this paper, we will study how contrastive methods (which are discriminative, rather than generative) can perform inference over times series. Ideally, we want representations of observations $x$ to be a sufficient statistic for temporal relationships (e.g., does $x'$ occur after $x$?) but need not retain other information about $x$ (e.g. the location of static objects).
This intuition motivates us to study how contrastive representation learning methods~\citep{oord2018representation, sohn2016improved, chen2020improved, tian2020contrastive, wu2018unsupervised} might be used to solve prediction and planning problems on time series data.
While prior works in {computer vision}~\citep{chen2020simple, oord2018representation} and natural language processing (NLP)~\citep{mikolov2013efficient} often study the geometry of learned representations, our results show how geometric operations such as interpolation are related to inference.
Our analysis will focus on a regularized version of the symmetrized infoNCE objective~\citep{radford2021learning}, generating positive examples by sampling pairs of observations from the same time series data.
We will study how representations learned in this way can facilitate two inference questions: prediction and planning.\footnote{Following prior work~\citep{botvinick2012planning, attiasPlanningProbabilisticInference2003c}, we will use \emph{planning} to refer to the problem of inferring intermediate states, not to refer to an optimal control problem.}
As a stepping stone, we will build upon prior work~\citep{wang2020understanding} to show that regularized contrastive learning should produce representations whose marginal distribution is an isotropic Gaussian distribution.

The main contribution of this paper is to demonstrate how intermediate and future time steps in a time series can be inferred easily using contrastive representations. This inference problem captures a number of practical tasks: interpolation, in-filling, and even planning and control, where the intermediate steps represent states between a stand and goal. While ordinarily these problems require an iterative inference or optimization procedure, with contrastive representations this can be done simply by inverting a low-dimensional matrix. In one special case, inference will correspond to linear interpolation.
Our first step is to prove that, under certain assumptions, the distribution over future representations has a Gaussian distribution, with a mean that is a linear function of the initial state representation (\cref{lemma:prediction}).
This paves the main to our main result (\cref{thm:planning}): \textit{given an initial and final state, we show that the posterior distribution over an intermediate state representations also follows a Gaussian distribution.}
Said in other words, the representations follow a Gauss-Markov chain,\footnote{This probabilistic model is equivalent to a discretized Ornstein-Uhlenbeck process~\citep{uhlenbeck1930theory} and is also known as an AR(1) model~\citep[Eq.~3.1.16]{box2015time}.}
wherein any joint or conditional distribution can be computed by inverting a low-dimensional matrix~\citep{malioutov2006walk, weiss1999correctness} (See \cref{fig:intuition-as-model}).
In one special case, inference will correspond to linearly interpolating between the representations of an initial state and final state. \Cref{sec:experiments} provides numerical experiments.

\section{Related Work}
\label{sec:prior-work}

\paragraph{Representations for time-series data.}
In applications ranging from robotics to vision to NLP, users often want to learn representations of observations from time series data such that the \emph{spatial} arrangement of representations reflects the \emph{temporal} arrangement of the observations~\citep{oord2018representation,mikolov2013efficient,qianSpatiotemporalContrastiveVideo2021,eysenbach2020c}.
Ideally, these representations should retain information required to predict future observations and infer likely paths between pairs of observations.
Many approaches use an autoencoder, learning representations that retain the bits necessary to reconstruct the input observation, while also regularizing the representations to compressed or predictable~\citep{karamchetiLanguageDrivenRepresentationLearning2023,parkFinetuningPretrainedTransformers2021,devlinBERTPretrainingDeep2019,carrollUniMASKUnifiedInference2022a,zhu2020s3vae,chungRecurrentLatentVariable2015}.
A prototypical method is the sequential VAE~\citep{zhao2017towards}, which is computationally expensive to train because of the reconstruction loss, but is easy to use for inference.
Our work shares the aims of prior prior methods that attempt to linearize the dynamics of nonlinear systems~\citep{shu2020predictive, watter2015embed, banijamali2018robust, cui2020control, nguyen2021temporal, nguyen2020non}, including videos~\citep{goroshin2015learning, jayaraman2016slow}.
Our work aims to retain uncertainty estimates over predictions (like the sequential VAE) without requiring reconstruction.
Avoiding reconstruction is appealing \emph{practically} because it decreases the computational requirements and number of hyperparameters; and \emph{theoretically} because it means that representations only need to retain bits about temporal relationships and not about the bits required to reconstruct the original observation.

\paragraph{Contrastive Learning.} Contrastive learning methods circumvent reconstruction by learning representations that merely classify if two events were sampled from the same joint distribution~\citep{gutmannNoisecontrastiveEstimationNew2010,chenExploringSimpleSiamese2020a,radford2021learning}. When applied to representing states along trajectories, contrastive representations learn to classify whether two points lie on the same trajectory or not \citep{oord2018representation,sermanet2018time,eysenbach2022contrastive,qianSpatiotemporalContrastiveVideo2021,xuXSkillCrossEmbodiment2023}.
Empirically, prior work in computer vision and NLP has observed that contrastive learning acquires representations where interpolation between representations corresponds to changing the images in semantically meaningful ways~\citep{wiskottSlowFeatureAnalysis2002, yanSemanticsGuidedRepresentationLearning2021,oringAutoencoderImageInterpolation2020,chenHomomorphicLatentSpace2019,liuDataAugmentationLatent2018, mikolov2013efficient}.

Our analysis will be structurally similar to prior theoretical analysis on explaining why word embeddings can solve analogies~\citep{levy2014linguistic, arora2016latent, allen2019analogies}. Our work will make a Gaussianity assumption similar to~\citet{arora2016latent} and our Markov assumption is similar to the random walks analyzed in~\citet{arora2016latent, hashimoto2016word}. Our paper builds upon and extends these results to answer questions such as: ``what is the distribution over future observations representations?'' and ``what is the distribution over state (representations) that would occur on the path between one observation and another?'' While prior work is primarily aimed at explaining the good performance of contrastive word embeddings (see, e.g.,~\citep{arora2016latent}), we are primarily interested in showing how similar contrastive methods are an effective tool for inference over high-dimensional time series data.
Our analysis will show how representations learned via temporal contrastive learning (i.e., without reconstruction) are sufficient statistics for inferring future outcomes and can be used for performing inference on a graphical model (a problem typically associated with generative methods).

\paragraph{Goal-oriented decision making.}
Much work on time series representations is done in service of learning goal-reaching behavior, an old problem~\citep{newell1959report, laird1987soar} that has received renewed attention in recent years~\citep{chen2021decision, chane2021goal, colas2021intrinsically, yang2022rethinking, ma2022far, schroecker2020universal, janner2021offline,hejna2023distance}.
Some of the excitement in goal-conditioned RL is a reflection of the recent success of self-supervised methods in computer vision~\citep{rombach2022high} and NLP~\citep{OpenAI2023GPT4TR}.
Our analysis will study a variant of contrastive representation learning proposed in prior work for goal-conditioned RL~\citep{eysenbach2022contrastive,sermanet2018time}.
These methods are widespread, appearing as learning objectives for learning value functions~\citep{eysenbach2020c,eysenbach2022contrastive,zhengStabilizingContrastiveRL2023,planningVisual, agarwal2019reinforcement, ma2022vip,LIV,nair2022r3m,Liu2022MetricRN,Wang2023OptimalGR}, as auxiliary objectives~\citep{schwarzer2020data, tang2022understanding, nair2022r3m, stooke2021decoupling, bharadhwaj2022information, Anand2019UnsupervisedSR, Castro2021MICoIR}, in objectives for model-based RL~\citep{shu2020predictive, ghugare2022simplifying, Allen2021LearningMS,mazoureContrastiveValueLearning2022}, and in exploration methods~\citep{Guo2022BYOLExploreEB, Du2021CuriousRL}.
Our analysis will highlight connections between these prior methods, the classic successor representation~\citep{Dayan1993ImprovingGF,barreto2017successor}, and probabilistic inference.

\paragraph{Planning.}
Planning lies at the core of many RL and control methods, allowing methods to infer the sequence of states and actions that would occur if the agent navigated from one state to a goal state.
While common methods such as PRM~\citep{kavraki1996probabilistic} and RRT~\citep{lavalle2001rapidly} focus on building random graphs, there is a strong community focusing on planning methods based on probabilistic inference~\citep{attiasPlanningProbabilisticInference2003c,thijssenPathIntegralControl2015,williamsModelPredictivePath2015}. The key challenge is scaling to high-dimensional settings. While semi-parametric methods make progress on this problem this limitation through semi-parametric planning~\citep{flap,eysenbachSearchReplayBuffer2019, zhang2021c}, it remains unclear how to scale any of these methods to high-dimensional settings when states do not lie on a low-dimensional manifold.
Our analysis will show how contrastive representations may lift this limitation, with experiments validating this theory on 39-dimensional and 46-dimensional tasks.

\section{Preliminaries}
\label{sec:prelims}

Our aim is to learn representations of time series data such that the spatial arrangement of representations corresponds to the temporal arrangement of the underlying data: if one example occurs shortly after another, then they should {be mapped to} similar representations.
This problem setting arises in many areas, including video understanding and reinforcement learning. To define this problem formally, we will define a Markov process with states $x_t$ indexed by time $t$:\footnote{This can be extended to \emph{controlled} Markov processes appending the previous action to the observations.}
$p(x_{1:T} \mid x_0) = \prod_{t=0}^T p(x_{t+1} \mid x_t)$.
The dynamics $p(x_{t+1} \mid x_t)$ tell us the immediate next state, and we can define the distribution over states $t$ steps in the future by marginalizing over the intermediate states,
$p_t(x_t \mid x_0) = \int p(x_{1:t} \mid x_0) \dx_{1:t-1}$.
A key quantity of interest will be the $\gamma$-discounted state occupancy measure, which corresponds to a time-averaged distribution over future states:
\begin{equation}
    p_{t+}(x_{t+} = x) = (1 - \gamma) \sum_{t=0}^\infty \gamma^t p_t(x_t = x).\label{eq:occupancy}
\end{equation}

\paragraph{Contrastive learning.}
Our analysis will focus on applying contrastive learning to a particular data distribution.
Contrastive learning~\citep{gutmannNoisecontrastiveEstimationNew2010,oord2018representation,aroraTheoreticalAnalysisContrastive2019} acquires representations using ``positive'' pairs $(x, x^+)$ and ``negative'' pairs $(x, x^-)$.
While contrastive learning typically learns just one representation, we will use two different representation for the two elements of the pair; that is, our analysis will use terms like $\phi(x)$, $\psi(x^+)$ and $\psi(x^-)$.
We assume all representations lie in $\mathbb{R}^k$.

The aim of contrastive learning is to learn representations such that positive pairs have similar representations ($\phi(x) \approx \psi(x^+)$) while negative pairs have dissimilar representations ($\phi(x) \neq \psi(x^-)$).
Let $p(x,x^+)$
be the joint distribution over positive pairs (i.e., $(x, x^+) \sim p(x,x^+)$).
We will use the product of the marginal distributions to sample negative pairs ($(x, x^-) \sim p(x)p(x)$).
Let $B$ be the batch size, and note that the positive samples $x_{j}^{+}$ at index $j$ in the batch serve as \emph{negatives} for $x_{i}$ for any $i\neq j$.
Our analysis is based on the infoNCE objective without resubstitution~\citep{sohn2016improved, oord2018representation}:
\begin{align}
    \max_{\phi(\cdot), \psi(\cdot)} \E_{\{(x_{i}, x_{i}^+)\}_{i=1}^{B} \sim p(x,x^+)} \Biggl[
        \sum_{i=1}^{B}
        \log \tfrac{e^{-\frac{1}{2} \|\phi(x_{i}) - \psi(x_{i}^+)\|_2^2}}{\sum_{j\neq i} e^{-\frac{1}{2} \|\phi(x_{i}) - \psi(x_j^+)\|_2^2}} +\log \tfrac{e^{-\frac{1}{2} \|\phi(x_{i}) - \psi(x_{i}^+)\|_2^2}}{\sum_{j\neq i} e^{-\frac{1}{2} \|\phi(x_{j}) - \psi(x_i^+)\|_2^2}} \Biggr] \label{eq:info-nce}
\end{align}

We will use the symmetrized version of this objective~\citep{radford2021learning}, where the denominator is the sum across rows of a logits matrix and once where it is a sum across the.

While contrastive learning is typically applied to an example $x$ and an augmentation $x^+ \sim p(x \mid x)$ of that same example (e.g., a random crop), we will follow prior work~\citep{sermanet2018time, oord2018representation} in using the time series \emph{dynamics} to generate the positive pairs, so $x^+$ will be an observation that occurs temporally after $x$.
While our experiments will sample positive examples from the discounted state occupancy measure ($x^+ \sim p_{t+}(x_{t+} \mid x)$) in line with prior work~\citep{eysenbach2022contrastive}, our analysis will also apply to different distributions (e.g., always sampling a state $k$ steps ahead).

While prior work typically constrains the representations to have a constant norm (i.e., to lie on the unit hypersphere)~\citep{oord2018representation}, we will instead constrain the \emph{expected} norm of the representations is bounded, a difference that will be important for our analysis:
\begin{equation}
    \tfrac{1}{k}\E_{p(x)}\left[\|\psi(x)\|_2^2 \right] \le c. \label{eq:norm}
\end{equation}

Because the norm scales with the dimension of the representation, we have scaled down the left side by the representation dimension, $k$. In practice, we will impose this constraint by adding a regularization term $\lambda \E_{p(x)}\left[\|\psi(x)\|_2^2 \right]$ to the infoNCE objective (\refas{Eq.}{eq:info-nce}) and dynamically tuning the weight $\lambda$ via dual gradient descent.

\subsection{Key assumptions}
\label{sec:assumptions}

This section outlines the two key assumptions behind our analysis, both of which have some theoretical justification. Our main assumption examines the distribution over representations: %
\makerestatable
\begin{assumption} \label{asm:marginal}
    Regularized, temporal contrastive learning acquires representations whose marginal distribution representations $p(\psi) \triangleq \int p(x)\1(\psi(x) = \psi) \dx$
    is an isotropic Gaussian distribution:
    \begin{equation}
        p(\psi) = \gN(\psi; \mu = 0, \sigma = c \cdot I).
        \label{eq:gaussian}
    \end{equation}
\end{assumption}
In \cref{appendix:gaussian} we extend prior work~\citep{wang2020understanding} provide some theoretical intuition for why this assumption should hold: namely, that the isotropic Gaussian is the distribution that maximizes entropy subject to an expected L2 norm constraint (\refas{Eq.}{eq:norm})~\citep{shannon1948mathematical, jaynesInformationTheoryStatistical1957,conrad2010probability}.
Our analysis also assumes that the learned representations converge to the theoretical minimizer of the infoNCE objective:
\makerestatable
\begin{assumption} \label{asm:symmetrized}
    Applying contrastive learning to the symmetrized infoNCE objective results in representations that encode a probability ratio:
    \begin{equation}
        e^{-\frac{1}{2}\|\phi(x_0) - \psi(x)\|_2^2} = \frac{p_{t+}(x_{t+} = x \mid x_0)}{p(x) C}.
        \label{eq:bayes-opt}
    \end{equation}
\end{assumption}
This assumption holds under ideal conditions~\citep{Ma2018NoiseCE, poole2019variational} (see \cref{appendix:symmetrized}),\footnote{While the result of~\citet{Ma2018NoiseCE} has $C(x)$ depending on $x$, the symmetrized version~\citep{radford2021learning} removes the dependence on $x$.}
but we nonetheless call this an ``assumption'' because it may not hold in practice due to sampling and function approximation error.
This assumption means the learned representations are sufficient statistics for predicting the probability (ratio) of future states: these representations must retain all the information pertinent to reasoning about \emph{temporal} relationships, but need not retain information about the precise contents of the observations. As such, they may be much more compressed than representations learned via reconstruction.

Combined, these assumptions will allow us to express the distribution over sequences of representations as a Gauss-Markov chain. The denominator in \cref{asm:symmetrized}, $p(x)$, may have a complex distribution, but \cref{asm:marginal} tells us that the distribution over \emph{representations} has a simpler form. This will allow us to rearrange \cref{asm:symmetrized} to express the conditional distribution over representations as the product of two Gaussian likelihoods. Note that the left hand side of \cref{asm:symmetrized} already looks like a Gaussian likelihood.

\section{Contrastive Representations Make Inference Easy}
\label{sec:method}

In this section, our main result will be to show how representations learned by (regularized) contrastive learning are distributed according to a Gauss-Markov chain, making it straightforward to perform inference (e.g., planning, prediction) over these representations. Our proof technique will
combine (known) results about Gaussian distributions with (known) results about contrastive learning. %
We start by discussing an important choice of parametrization (\cref{sec:parametrization}) that facilitates prediction (\cref{sec:prediction}) before presenting the main result in \cref{sec:planning}.

\subsection{A Parametrization for Shared Encoders}
\label{sec:parametrization}

\begin{wrapfigure}[12]{R}{0.3\textwidth}
    \centering
    \vspace{-3.5em}
    \includegraphics[trim={0 8cm 9cm 1cm},clip,width=\linewidth]{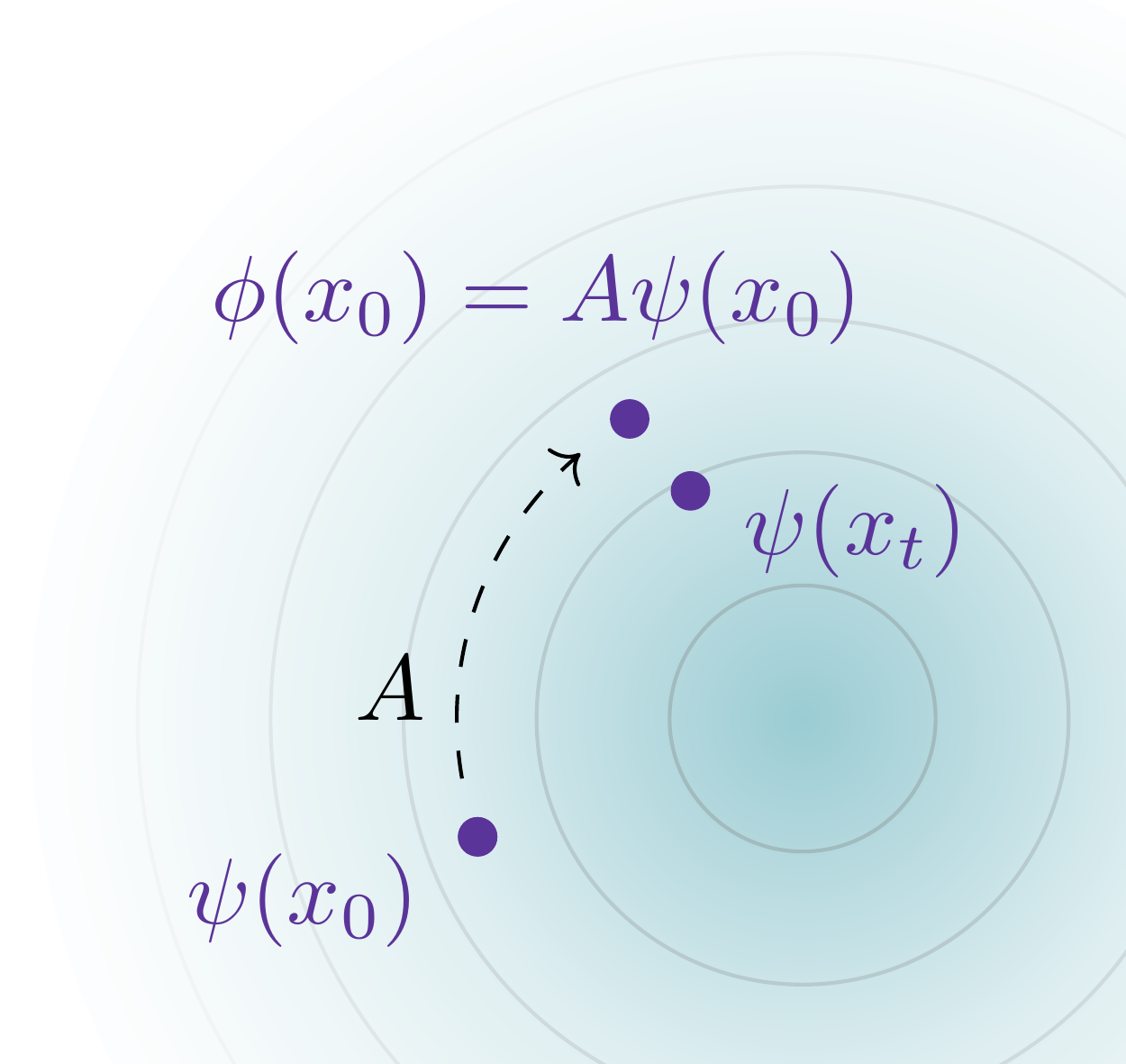}
    \vspace{-1em}
    \caption{A parametrization for temporal contrastive learning.}
    \label{fig:parametrization}
\end{wrapfigure}
This section describes the two encoders ($\psi(\cdot), \phi(\cdot)$) to compute representations of $x$ and $x^+$.
While prior work in computer vision and NLP literature use the same {encoder} for both $x$ and $x^+$, this decision does not make sense for many time-series data as it would imply that our prediction for $p(x_t \mid x_0)$ is the same as our prediction for $p(x_0 \mid x_t)$. However, the difficulty of transiting from $x_0$ to $x_t$ (e.g., climbing to the peak of a mountain) might be more difficult than the reverse (e.g., sledding down a mountain). %
Our proposed parametrization will handle this asymmetry. %

We will treat the encoder $\psi(\cdot)$ as encoding the contents of the state.
We will additionally learn a  matrix $A$ so that the function $\psi \mapsto A \psi$ corresponds to a (multi-step) prediction of the future representation. To map this onto contrastive learning, we will use $\phi(x) \triangleq A \psi(x)$ as the encoder for the initial state. One way of interpreting this encoder is as an additional linear projection applied on top of $\psi(\cdot)$, a design similar to those used in other areas of contrastive learning~\citep{chenExploringSimpleSiamese2020a}.
Once learned, we can use these {encoders} to answer questions about prediction (\cref{sec:prediction}) and planning (\cref{sec:planning}).

\subsection{Representations Encode a Predictive Model}
\label{sec:prediction}

\begin{wrapfigure}[13]{R}{0.3\linewidth}
    \centering
    \vspace*{-2em}
    \includegraphics[width=\linewidth,trim={2.5cm 6cm 6.5cm 1cm},clip]{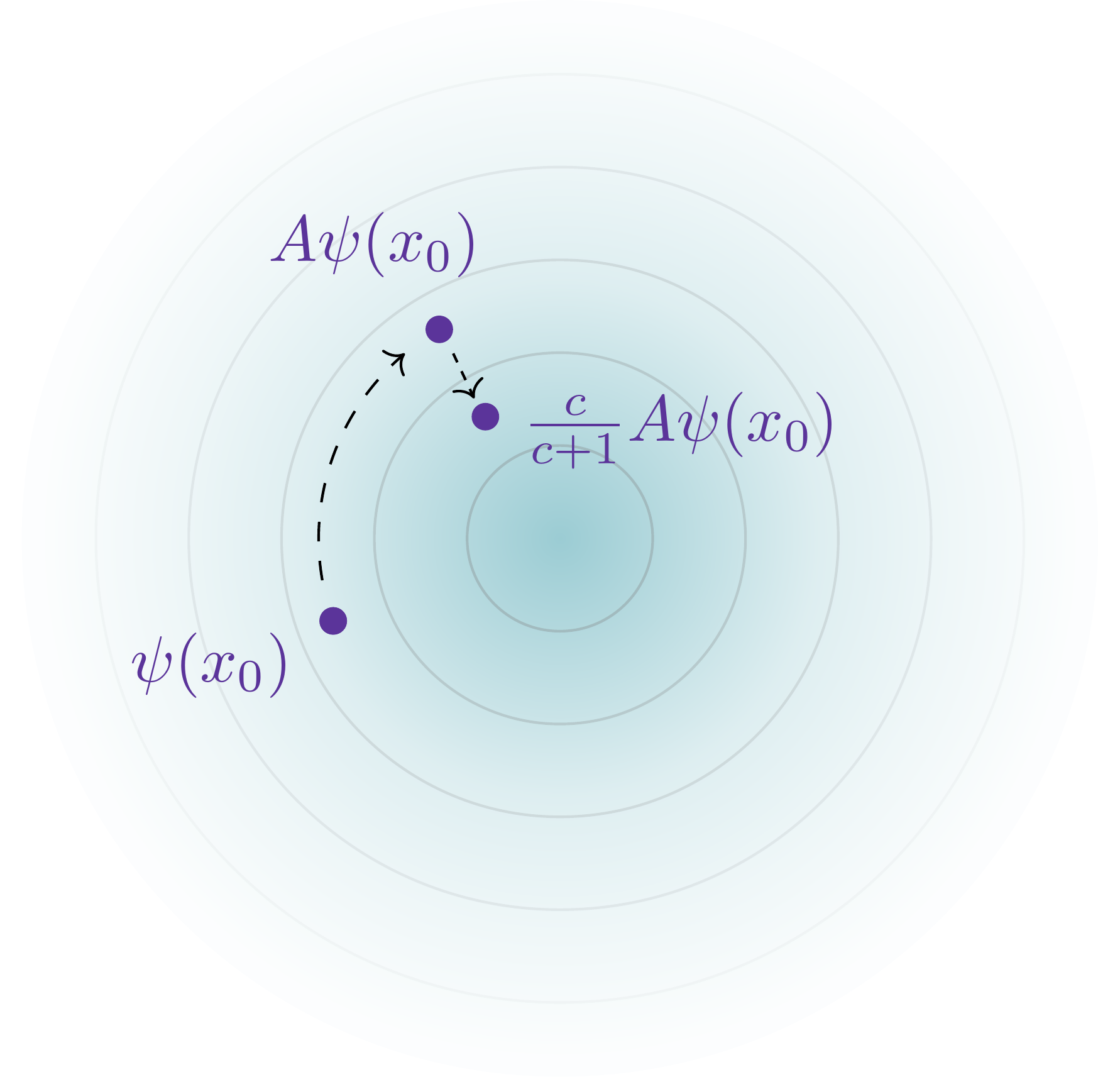}
    \vspace*{-1em}
    \caption{Predicting representations of future states.}
    \label{fig:prediction}
\end{wrapfigure}

Given an initial state $x_0$, what states are likely to occur in the future? Answering this question directly in terms of high-dimensional states is challenging, but our learned representations provide a straightforward answer.
Let $\psi_0 = \psi(x_0)$ and $\psi_{t+} = \psi(x_{t+})$ be random variables representing the representations of the initial state and a future state. Our aim is to estimate the distribution over these future representations, $p(\psi_{t+} \mid \psi_0)$. We will show that the learned representations encode this distribution.

\makerestatable
\begin{lemma}
    \label{lemma:prediction}
    Under the assumptions from \cref{sec:prelims}, the distribution over representations of future states follows a Gaussian distribution with mean parameter given by the initial state representation:
    \begin{equation}
        p(\psi_{t+} = \psi \mid \psi_0) = \gN\Bigl(\mu = \frac{c}{c+1} A \psi_0, \Sigma = \frac{c}{c+1}I\Bigr). \label{eq:marginal}
    \end{equation}
\end{lemma}

The main takeaway here is that the distribution over future representations has a convenient, closed form solution. The representation norm constraint, $c$, determines the shrinkage factor $\frac{c}{c+1} \in [0, 1)$; highly regularized settings (small $c$) move the mean closer towards the origin and decrease the variance, as visualized in \cref{fig:prediction}. Regardless of the constraint $c$, the predicted mean is a linear function $\psi \mapsto \frac{c}{c+1}A \psi$.
The proof is in \cref{appendix:prediction}. The proof technique is similar to that of the law of the unconscious statistician.

\subsection{Planning over One Intermediate State}
\label{sec:planning}

We now show how these representations can be used for a specific type of planning: given an initial state $x_0$ and a future state $x_{t+}$, infer the representation of an intermediate ``waypoint'' state $x_w$. The next section will extend this analysis to inferring the entire sequence of intermediate states. We assume $x_0 \rightarrow x_w \rightarrow x_{t+}$ form a Markov chain where $x_w \sim p(x_{t+} \mid x_0 = x_0)$ and $x_{t+} \sim p(x_{t+} \mid x_0 = x_w)$ are both drawn from the discounted state occupancy measure (\refas{Eq.}{eq:occupancy}). Let random variable $\psi_w = \psi(x_w)$ be the representation of this intermediate state.
Our main result is that the posterior distribution over waypoint \emph{representations} has a closed form solution in terms of the initial state representation and future state representation:
\makerestatable
\begin{theorem} \label{thm:planning}
    Under Assumptions 1 and 2, the posterior distribution over waypoint representations is a Gaussian whose mean and covariance are linear functions of the initial and final state representations:
    \begin{align*}
        p(\psi_w \mid \psi_0, \psi_{t+}) = \mathcal{N}\Bigl(
        \psi_w;
        \mu & \,{\color{black}= \Sigma(A^T \psi_{t+} + A \psi_0)},
        \Sigma^{-1} \,{\color{black} = \tfrac{c}{c+1}A^T A + \tfrac{c+1}{c}I}
        \Bigr).
    \end{align*}
\end{theorem}
The proof (\cref{appendix:planning}) uses the Markov property together with \cref{lemma:prediction}.
The main takeaway from this lemma is that the posterior distribution takes the form of a simple probability distribution (a Gaussian) with parameters that are linear functions of the initial and final representations.

We give three examples to build intuition:
\paragraph{Example 1:} $A = I$ and the $c$ is very large (little regularization). Then, the covariance is $\Sigma^{-1} \approx 2I$ and the mean is the simple average of the initial and final representations $\mu \approx \frac{1}{2} (\psi_0 + \psi_{t+})$. In other words, the waypoint representation is the midpoint of the line $\psi_0 \rightarrow \psi_{t+}$.
\paragraph{Example 2:} $A$ is a rotation matrix and $c$ is very large. Rotation matrices satisfy $A^T = A^{-1}$ so the covariance is again $\Sigma^{-1} \approx 2I$. As noted in \cref{sec:prediction}, we can interpret $A \psi_0$ as a  \emph{prediction} of which representations will occur after $\psi_0$. Similarly, $A^{-1} \psi_{t+} = A^T \psi_{t+}$ is a prediction of which representations will occur before $\psi_{t+}$. \Cref{thm:planning} tells us that the mean of the waypoint distribution is the simple average of these two predictions, $\mu \approx \frac{1}{2}(A^T \psi_{t+} + A \psi_0)$.
\paragraph{Example 3:} $A$ is a rotation matrix and $c = 0.01$ (very strong regularization). In this case $\Sigma^{-1} = \frac{0.01}{0.01 + 1}A^TA + \frac{0.01 + 1}{0.01}I \approx 100I$, so $\mu \approx \frac{1}{100}(\psi_0 + \psi_{t+}) \approx 0$. Thus, in the case of strong regularization, the posterior concentrates around the origin.

\subsection{Planning over Many Intermediate States}

This section extends the analysis to multiple intermediate states. Again, we will infer the posterior distribution of the representations of these intermediate states, $\psi_{w_1}, \psi_{w_2}, \cdots$. We assume that these states form a Markov chain.

\makerestatable
\begin{theorem} \label{thm:multistep}
    Given observations from a Markov chain $x_0 \rightarrow x_1 \cdots x_{t+}$, the joint distribution over representations is a Gaussian distribution. Using $\psi_{1:n} = \begin{pmatrix} \psi_{w_1}, \cdots, \psi_{w_n} \end{pmatrix}$ to denote the concatenated representations of each observation, we can write this distribution as
    \begin{align*}
        p(\psi_{1:n}) \propto \exp\bigl( -\tfrac{1}{2}\psi_{1:n}^T \Sigma^{-1} \psi_{1:n} + \eta^T \psi_{1:n} \bigr),
    \end{align*}
    where $\Sigma^{-1}$ is a tridiagonal matrix
    \begin{align*}
        \Sigma^{-1} & = \left(\begin{matrix}
                                      \frac{c}{c+1}A^TA + \frac{c+1}{c}I & -A^T & & \\
                                      -A & \frac{c}{c+1}A^TA + \frac{c+1}{c}I & -A^T & \\[-2ex]
                                      &   &  & \ddots \\
                                  \end{matrix}\right) \quad  \text{and}\;
        \eta  = \left(\begin{matrix}
                              A \psi_0 \\
                              0 \\[-3pt]
                              \vdots \\
                              A^T \psi_{t+}
                          \end{matrix}\right).
    \end{align*}
\end{theorem}
This distribution can be written in the canonical parametrization as $\Sigma = \Lambda^{-1}$ and $\mu = \Sigma \eta$.
Recall that Gaussian distributions are closed under marginalization. Thus, once in this canonical parametrization, the marginal distributions can be obtained by reading off individual entries of these parameters:
\begin{equation*}
    p(\psi_i \mid \psi_0, \psi_{t+}) = \gN\left(\psi_i; \mu_i = (\Sigma \eta)^{(i)}, \Sigma_i = (\Lambda^{-1})^{(i, i)}\right).
\end{equation*}
The key takeaway here is that this posterior distribution over waypoints is Gaussian, and it has a closed form expression in terms of the initial and final representations (as well as regularization parameter $c$ and the learned matrix $A$).

In the general case of $n$ intermediate states, the posterior distribution is
\begin{align*}
    p(\psi_{w_1}\cdots\psi_{w_n} \mid \psi_0, \psi_{t+})
     & \propto e^{-\frac{1 + \frac{1}{c}}{2} \sum_{i=1}^n \|\frac{c}{c+1}A \psi_{w_i} - \psi_{w_{i+1}}\|_2^2},
\end{align*}
where $\psi_{w_0} = \psi_0$ and $\psi_{w_{n+1}} = \psi_{t+}$.
This corresponds to a chain graphical model with edge potentials $f(\psi, \psi') = e^{-\frac{1 + \frac{1}{c}}{2} \|\frac{c}{c+1}A \psi - \psi'\|_2^2}$.

\paragraph{Special case.} To build intuition, consider the special case where $A$ is a rotation matrix and $c$ is very large, so $\frac{c}{c+1}A^T A + \frac{c+1}{c} \approx 2I$. In this case, $\Sigma^{-1}$ is a (block) second difference matrix~\citep{higham2022what}:
\begin{equation*}
    \Sigma^{-1} = \left(\begin{smallmatrix}
            2I & -I & & \\
            -I & 2I & -I & \\[-1ex]
            & & & \ddots
        \end{smallmatrix}\right).
\end{equation*}
The inverse of this matrix has a closed form solution~\citep[Pg.~471]{newman1958evaluation}, allowing us to obtain the mean of each waypoint in closed form:
\begin{equation}
    \mu_i = (1 - \lambda(i)) A \psi_0 + \lambda(i) A^T \psi_{t+}, \label{eq:special-case}
\end{equation}
where $\lambda(i) = \tfrac{i}{n+1}$.
Thus, each posterior mean is a convex combination of the (forward prediction from the) initial representation and the (backwards prediction from the) final representation. When $A$ is the identity matrix, the posterior mean is simple linear interpolation between the initial and final representations!

\begin{figure*}
    \centering
    \includegraphics[width=\linewidth]{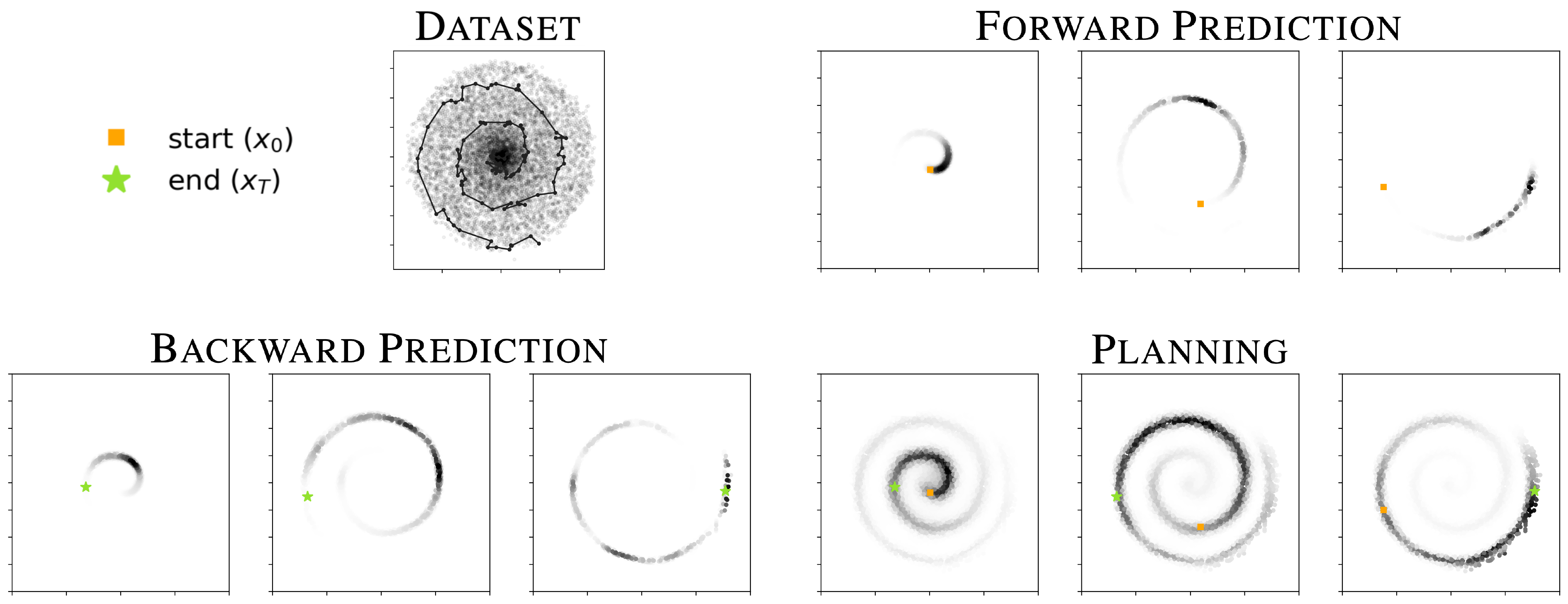}
    \caption{\textbf{Numerical simulation of our analysis.}
        \figtopleft \, Toy dataset of time-series data consisting of many outwardly-spiraling trajectories. We apply temporal contrastive learning to these data.
        \figtopright \, For three initial observations (\raisebox{-0.18ex}{\color{orange}$\blacksquare$}), we use the learned representations to predict the distribution over future observations. Note that these distributions correctly capture the spiral structure.
        \figbottomleft \, For three observations ({\color{lime}$\star$}), we use the learned representations to predict the distribution over preceding observations.
        \figbottomright \, Given an initial and final observation, we plot the inferred posterior distribution over the waypoint (\cref{sec:planning}). The representations capture the shape of the distribution.
    }
    \label{fig:spiral}
\end{figure*}

\section{Numerical Simulation}
\label{sec:experiments}

We include several didactic experiments to illustrate our results. %
All results and figures can be reproduced by running \texttt{make} in the source code: \URL. The expected compute time is a few hours on a A6000 GPU. Figures in this section show error across different training and dataset split seeds.

\subsection{Synthetic Dataset}
To validate our analysis, we design a time series task with 2D points where inference over intermediate points (i.e., in-filling) requires nonlinear interpolation.
\cref{fig:spiral} \textit{(Top Left)} shows the dataset of time series data, starting at the origin and spiraling outwards, with each trajectory using a randomly-chosen initial angle.
We applied contrastive learning with the parametrization in \cref{sec:parametrization} to these data and used the learned representations to solve prediction and planning problems (see \cref{fig:spiral} for details).
Note that these predictions correctly handle the nonlinear structure of these data \-- states nearby the initial state in Euclidean space that are not temporally adjacent are assigned low likelihood.

\begin{figure}[htb]
    \centering
    \ignorespaces
    \begin{minipage}[m]{.25\linewidth}
        \centering
        \includegraphics[width=.9\linewidth]{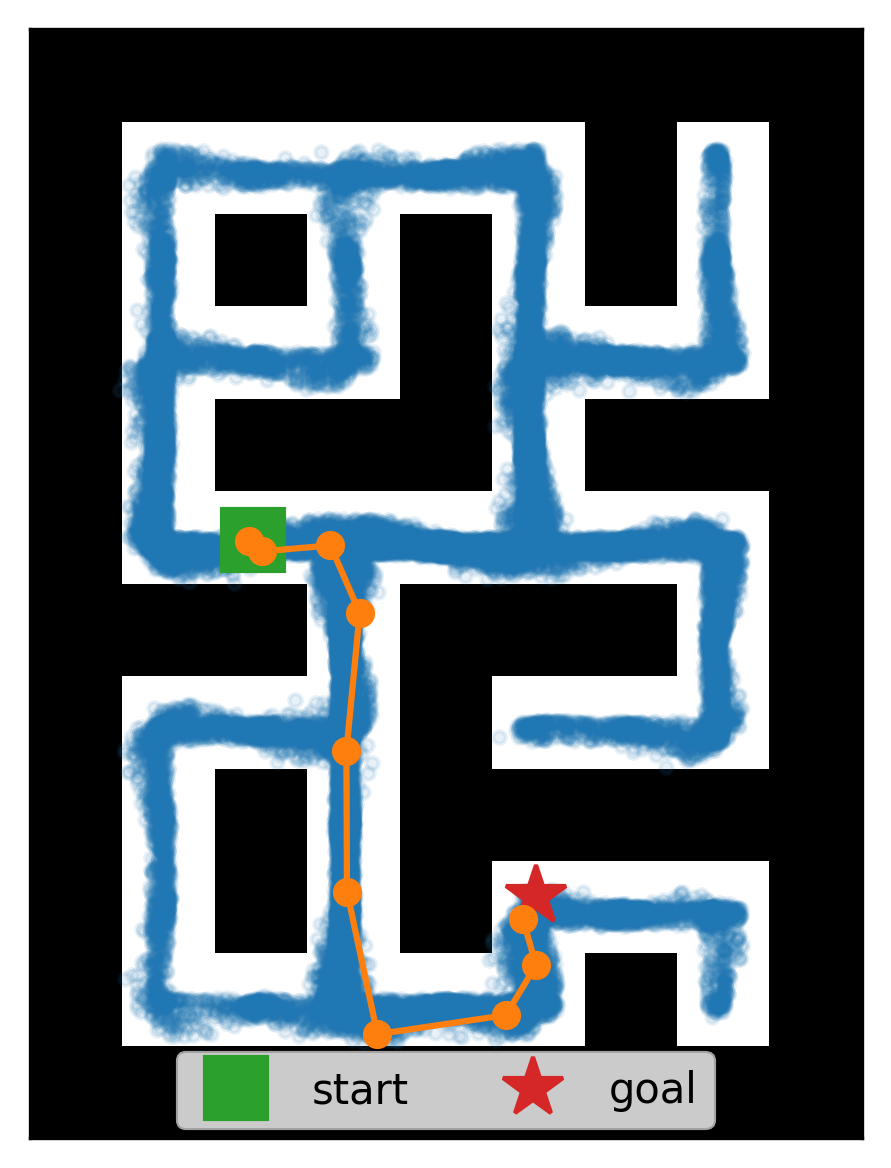}%
    \end{minipage}
    \hspace{1em}
    \begin{minipage}[m]{.6\linewidth}
        \centering
        \includegraphics{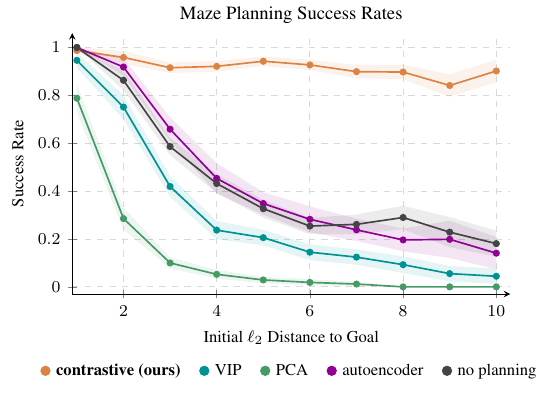}
    \end{minipage}
    \caption{Using inferred paths over our contrastive representations for control boosts success rates by $4.5\times$ on the most difficult goals ($18\% \rightarrow 84\%$). Alternative representation learning techniques fail to improve performance when used for planning. }
    \label{fig:maze-results}
\end{figure}

\subsection{Solving Mazes with Inferred Representations}

Our next experiment studies whether the inferred representations are useful for solving a control task.
We took a 2d maze environment and dataset from prior work (\cref{fig:maze-results}, \textit{Left})~\citep{fu2020d4rl} and learned encoders from this dataset.
To solve the\nobreak\ maze, we\nobreak\ take the observation of the starting state and goal state, compute the representations of these states, and use the analysis in \cref{sec:planning} to infer the sequence of intermediate representations.
We visualize the results using a nearest neighbor retrieval (\cref{fig:maze-results}, \textit{Left}).
\Cref{fig:maze} contains additional examples.

Finally, we studied whether these representations are useful for control. We implemented a simple proportional controller for this maze. As expected, this proportional controller can successfully navigate to close goals, but fails to reach distant goals (\cref{fig:maze-results}, \textit{Right}). However, if we use the proportional controller to track a series of waypoints planned using our representations (i.e., the orange dots shown in \cref{fig:maze-results} \textit{(Left)}), the success rate increases by up to $4.5\times$.
To test the importance of \emph{nonlinear} representations, we compare with a ``PCA'' baseline that predicts waypoints by interpolating between the principal components of the initial state and goal state. The better performance of our method indicates the importance of doing the interpolation using representations that are \emph{nonlinear} functions of the input observations.
While prior methods learn representations to encode temporal distances, it is unclear whether these methods support inference via interpolation. To test this hypothesis, we use one of these methods (``VIP''~\citep{ma2022vip}) as a baseline. While the VIP representations likely encode similar bits as our representations, the better performance of the contrastive representations indicates that the VIP representations do not expose those bits in a way that makes planning easy.

\begin{figure*}
    \centering
    \begin{subfigure}[c]{0.22\linewidth}
        \centering
        \hspace*{0.5cm}\includegraphics[width=\linewidth]{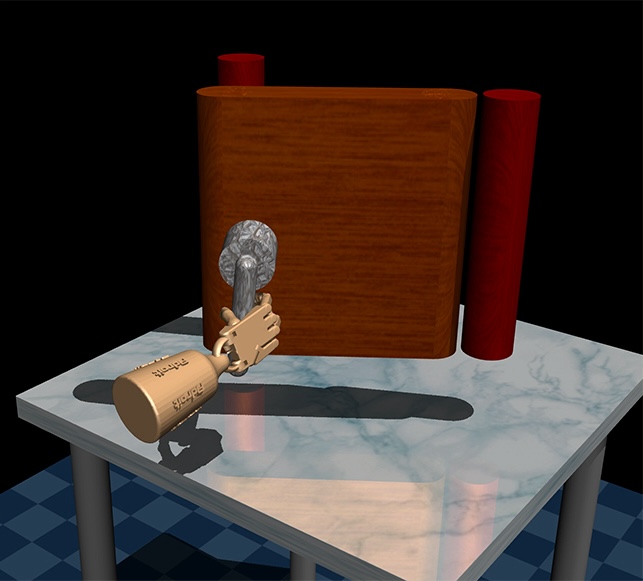}%
        \vspace*{3pt}%
    \end{subfigure}%
    \hspace*{1.5cm}%
    \begin{subfigure}[c]{0.7\linewidth}%
        \centering%
        \includegraphics[width=\linewidth]{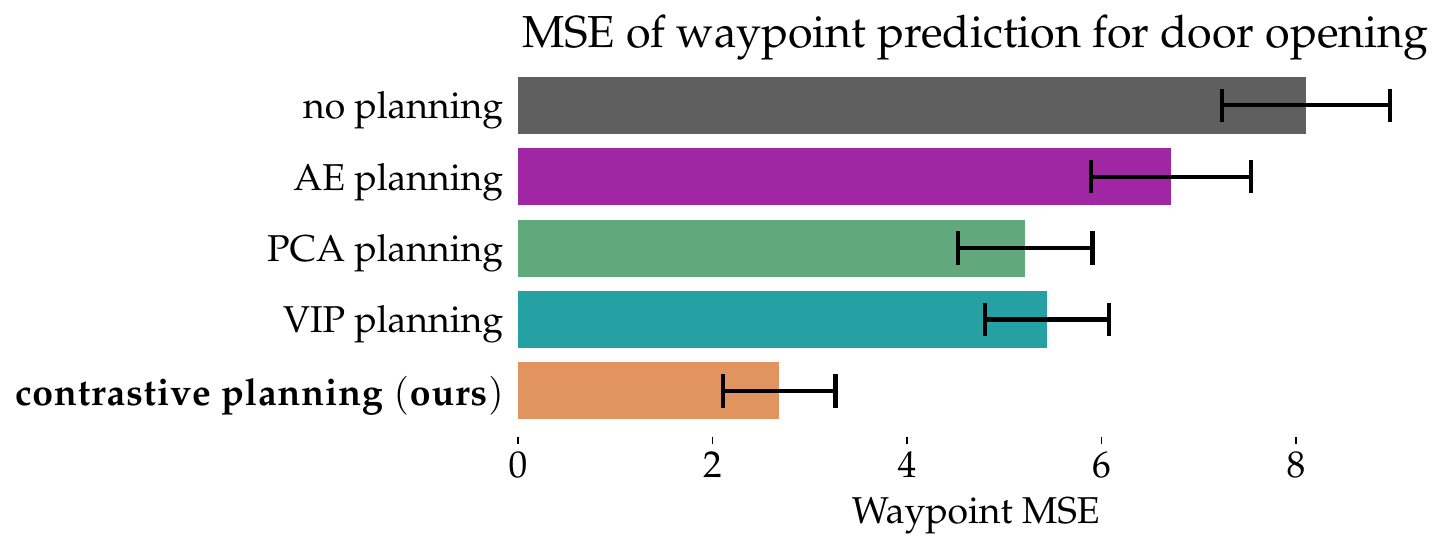}%
    \end{subfigure}\par
    \begin{subfigure}[c]{\linewidth}
        \vspace*{1ex}
        \centering
        \includegraphics[width=\linewidth]{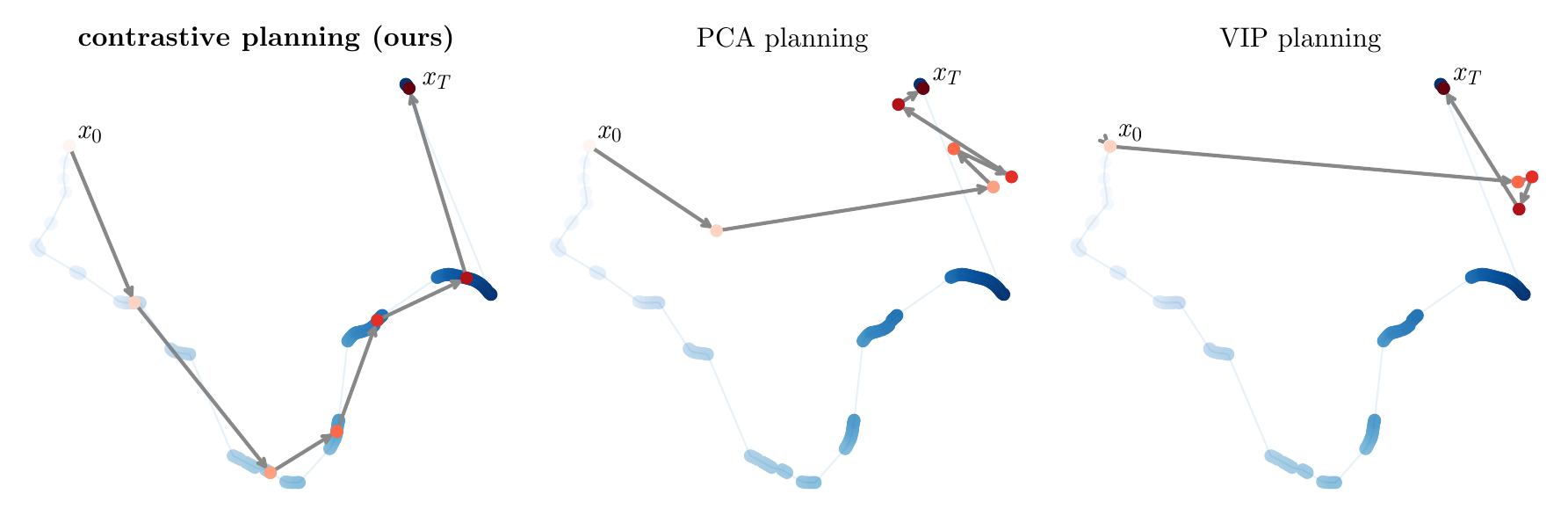}
    \end{subfigure}
    \caption{Planning for 39-dimensional robotic door opening.
        \figtopleft \, We use a dataset of trajectories demonstrating door opening from prior work~\citep{fu2020d4rl} to learn representations.
        \figtopright \, We use our method and three baselines to infer one intermediate waypoint between the first and last observation in a trajectory from a held-out validation set. Errors are measured using the mean squared error with the true waypoint observation; predicted representations are converted to observations using nearest neighbors on a validation set.
        \figbottom \, We visualize a TSNE~\citep{van2008visualizing} of the states along the sampled trajectory as blue circles, with the transparency indicating the index along the trajectory. The inferred plan is shown as red circles connected by arrows. Our method generates better plans than alternative representation learning methods (PCA, VIP).
    }
    \label{fig:door}
\end{figure*}

\subsection{Higher dimensional tasks}
\label{sec:high-dim}

In this section we provide preliminary experiments showing the planning approach in \cref{sec:method} scales to higher dimensional tasks. We used two datasets from prior work~\citep{fu2020d4rl}: \texttt{door-human-v0} (39-dimensional observations) and \texttt{hammer-human-v0} (46-dimensional observations).
After learning {encoders} on these tasks, we evaluated the inference capabilities of the learned representations. Given the first and last observation from a trajectory in a validation set, we use linear interpolation (see \refas{Eq.}{eq:special-case}) to infer the representation of five intermediate waypoint representations.

We evaluate performance in two ways. \textbf{Quantitatively}, we measure the mean squared error between each of the true waypoint observations and those inferred by our method. Since our method infers representations, rather than observations, we use a nearest-neighbor retrieval on a validation set so that we can measure errors in the space of observations. \textbf{Qualitatively}, we visualize the high-dimensional observations from the validation trajectory using a 2-dimensional TSNE~\citep{van2008visualizing} embedding, overlying the infer waypoints from our method; as before, we convert the representations inferred by our method to observations using nearest neighbors.

We compare with three alternative methods in \cref{fig:door}. To test the importance of representation learning, we first na\"ively interpolate between the initial and final observations (``no planning''). The poor performance of this baseline indicates that the input time series are highly nonlinear. Similarly, interpolating the principle components of the initial and final observations (``PCA'') performs poorly, again highlighting that the input time series is highly nonlinear and that our representations are doing more than denoising (i.e., discarding directions of small variation). The third baseline, ``VIP''~\citep{ma2022vip}, learns representations to encode temporal distances using approximate dynamic programming. Like our method, VIP avoids reconstruction and learns nonlinear representations of the observations. However, the results in \cref{fig:door} highlight that VIP's representations do not allow users to plan by interpolation.
The error bars shown in \cref{fig:door} (\emph{Top Right}) show the standard deviation over 500 trajectories sampled from the validation set. For reproducibility, we repeated this entire experiment on another task, the 46-dimensional \texttt{hammer-human-v0} from D4RL. The results, shown in Appendix \cref{fig:hammer}, support the conclusions above.
Taken together, these results show that our procedure for interpolating contrastive representations continues to be effective on tasks where observations have dozens of dimensions.

\section{Discussion}
\label{sec:conclusion}

Representation learning is at the core of many high-dimensional time-series modeling questions, yet how those representations are learned is often disconnected with the inferential task. The main contribution of this paper is to show how \emph{discriminative} techniques can be used to acquire compact representations that make it easy to answer inferential questions about time. The precise objective and parametrization we studied is not much different from that used in practice, suggesting that either our theoretical results might be adapted to the existing methods, or that practitioners might adopt these details so they can use the closed-form solutions to inference questions.
Our work may also have implications for studying the structure of learned representations.
While prior work often studies the geometry of representations as a post-hoc check, our analysis provides tools for studying \emph{when} interpolation properties are guaranteed to emerge, as well as \emph{how} to learn representations with certain desired geometric properties.

\paragraph{Limitations.}
Our analysis hinges on the two assumptions mentioned in \cref{sec:assumptions}, and it remains open how errors in those approximations translate into errors in our analysis.
One important open question is whether it is always possible to satisfy these assumptions using sufficiently-expressive representations.

\section*{Acknowledgments}
We thank Seohong Park, Gautam Reddy, Chongyi Zheng, and anonymous reviewers for feedback and discussions that shaped this project. This work was supported by Princeton Research Computing resources at Princeton University. This work was partially supported by AFOSR FA9550-22-1-0273.

\bibliographystyle{custom}
\bibsep=3pt
\bibliography{references}

\appendix

\allowdisplaybreaks
\textfloatsep 17pt plus 6.0pt minus 6.0pt
\floatsep 12pt plus 6.0pt minus 6.0pt

\onecolumn
\renewenvironment{figure*}{\figure}{\endfigure}

\section{Proofs}
\label{appendix:proofs}

This section contains the proofs omitted from the main text.

\def\ptrue{p^*}
\def\pdata{p_{\D}}
\def\D{\mathcal D}
\def\lalign{\mathcal L_{\text{align}}}
\def\luni{\mathcal L_{\text{uniform}}}
\def\l2#1{\left\|#1\right\|_2^2}
\def\goes{\raisebox{2px}{$\underset{\scriptscriptstyle\lalign\to0}\to$}}%

\subsection{Marginal Distribution over Representations is Gaussian}
\label{appendix:gaussian}

Recall \cref{asm:marginal}, which states that the marginal distribution over representations is Gaussian.
\restatetheorem{asm:marginal}

We will motivate this statement by connecting the optimal contrastive infoNCE objective to the maximum entropy marginal distribution over the representations.

The infoNCE objective (\refas{Eq.}{eq:info-nce}) can be decomposed into an alignment term and a uniformity term~\citep{wang2020understanding}, where the uniformity term can be simplified as follows:
\begin{align*}
     & \E_{x \sim p(x)} \left[ \log \E_{x^- \sim p(x)} \left[e^{-\frac{1}{2}\|A\psi(x^-) - \psi(x)\|_2^2} \right] \right]                      \\
     & = \frac{1}{N} \sum_{i=1}^N \log \bigg(\frac{1}{N-1} \sum_{j=1 \cdots N, j \neq i} e^{-\frac{1}{2}\|A\psi(x_i) - \psi(x_j)\|_2^2} \bigg) \\
     & = \frac{1}{N} \sum_{i=1}^N \log \bigg(\frac{1}{N-1} \sum_{j=1 \cdots N, j \neq i} \underbrace{\frac{1}{(2
    \pi)^{k/2}}e^{-\frac{1}{2}\|A\psi(x_i) - \psi(x_j)\|_2^2}}_{\gN(\mu = \psi(x_j); \Sigma = I)} \bigg) + \frac{k}{2} \log (2 \pi)            \\
     & = \frac{1}{N} \sum_{i=1}^N \log \hat{p}_\textsc{GMM}(\psi(x_i)) + \frac{k}{2} \log (2 \pi)                                              \\
     & = -\hat{\gH}[\psi(x)] + \frac{k}{2} \log (2 \pi).
\end{align*}

The derivation above extends that in~\citet{wang2020understanding} by considering a Gaussian distribution rather than a von Mises Fisher distribution.
We are implicitly making the assumption that the marginal distributions satisfy $p(x) = p(x^-)$.
This difference corresponds to our choice of using a negative squared L2 distance in the infoNCE loss rather than an inner product, a difference that will be important later in our analysis.
A second difference is that we do not use the resubstitution estimator (i.e., we exclude data point $x_i$ from our estimate of $\hat{p}_\text{GMM}$ when evaluating the likelihood of $x_i$), which we found hurt performance empirically.
The takeaway from this identity is that maximizing the uniformity term corresponds to maximizing (an estimate of) the entropy of the representations.

We next prove that the maximum entropy distribution with an expected L2 norm constraint is a Gaussian distribution. Variants of this result are well known~\citep{shannon1948mathematical, jaynesInformationTheoryStatistical1957,conrad2010probability}, but we include a full proof here for transparency.

\begin{lemma} \label{lemma:maxent}
    The maximum entropy distribution satisfying the expected L2 norm constraint in \cref{eq:norm} is a multivariate Gaussian distribution with mean $\mu = 0$ and covariance $\Sigma = c \cdot I$
\end{lemma}

\begin{proof}
    We start by defining the corresponding Lagrangian, with the second constraint saying that $p(x)$ must be a valid probability distribution.
    \begin{align*}
        \gL(p) = \gH_p[x] + \lambda_1 \left( \E_{p(x)}\left[\|x\|_2^2\right] - c \cdot k \right) + \lambda_2 \left( \int p(x) \dx - 1 \right)
    \end{align*}

    We next take the derivative w.r.t. $p(x)$:
    \begin{align*}
        \frac{\partial \gL}{\partial p(x)} = -p(x) / p(x) - \log p(x) + \lambda_1 \|x\|_2^2 + \lambda_2
    \end{align*}

    Setting this derivative equal to 0 and solving for $p(x)$, we get
    \begin{equation*}
        p(x) = e^{-1 + \lambda_2 + \lambda_1 \|x\|_2^2}.
    \end{equation*}

    We next solve for $\lambda_1$ and $\lambda_2$ to satisfy the constraints in the Lagrangian. Note that $x \sim \gN(\mu=0, \Sigma = c \cdot I)$ has an expected norm $\E[\|x\|_2^2] = c \cdot k$, so we must have $\lambda_1 = -\frac{1}{2c}$. We determine $\lambda_1$ as the normalizing constant for a Gaussian, finally giving us:
    \begin{equation*}
        p(x) = \frac{1}{(2 c\pi)^{k/2}} e^{\frac{-1}{2c} \|x\|_2^2}
    \end{equation*}
    corresponding to an isotropic Gaussian distribution with mean $\mu = 0$ and covariance $\Sigma = c \cdot I$.
\end{proof}

\subsection{Proof of \protect\cref{lemma:prediction}}
\label{appendix:prediction}
Below we present the proof of \cref{lemma:prediction}.

\restatetheorem{lemma:prediction}

\begin{proof}
    Our proof technique will be similar to that of the law of the unconscious statistician:
    \begingroup
    \begin{align*}
        p(&\psi_{t+} \mid \psi_0) \stackrel{(a)}{=} \frac{p(\psi_{t+}, \psi_0)}{\cancel{p(\psi_0)}}\propto \iint p(\psi_{t+}, x_{t+}, \psi_0, x_0)
            \dx_{t+} \dx_0 \\
        &\stackrel{(b)}{=} \iint p(\psi_{t+} \mid x_{t+}) p(\psi_0 \mid x_0) p(x_{t+} \mid x_0) p(x_0)
            \dx_{t+} \dx_0 \\
        &\stackrel{(c)}{\propto} \iint \1(\psi(x_{t+}) = \psi_{t+}) \1(\psi(x_0) = \psi_0) p(x_{t+})
            e^{-\frac{1}{2}\|A\psi(x_0)
            - \psi(x_{t+})\|_2^2} p(x_0) \dx_{t+} \dx_0 \\
        &\stackrel{(d)}{=} e^{-\frac{1}{2}\|A\psi_0 - \psi_{t+}\|_2^2} \iint \1(\psi(x_{t+}) =
            \psi_{t+})\1(\psi(x_0) = \psi_0)
            p(x_{t+}) p(x_0) \dx_{t+} \dx_0 \\
        &\stackrel{(e)}{=} e^{-\frac{1}{2}\|A\psi_0 - \psi_{t+}\|_2^2} \underbrace{\left(\int p(x_{t+})
            \1(\psi(x_{t+})) \dx_{t+}
            \right)\hspace*{-2pt}}_{p(\psi_{t+})}\hspace*{3pt}\underbrace{\hspace*{-2pt}\left(\int p(x_0) \1 (\psi(x_0) \dx_0 \right)}_{p(\psi_0)} \\
        &\stackrel{(f)}{\propto} e^{-\frac{1}{2}\|A\psi_0 - \psi_{t+}\|_2^2} e^{-\frac{1}{2c}
            \|\psi_{t+}\|_2^2} e^{-\frac{1}{2c} \|\psi_0\|_2^2}
            \\
        &\stackrel{(g)}{\propto} e^{-\frac{1 + \frac{1}{c}}{2}\bigl\|{\frac{1}{1 + \frac{1}{c}}A \psi_0 -
            \psi_{t+}}\bigr\|_2^2} \\
        &\propto \gN \left(\psi_{t+}; \mu = \tfrac{c}{c + 1} A \psi_0, \Sigma = \tfrac{c}{c+1}I \right).
    \end{align*}
    \endgroup%

    In \emph{(a)} we applied Bayes' Rule and removed the denominator, which is a constant w.r.t. $\psi_{t+}$.
    In \emph{(b)} we factored the joint distribution, noting that $\psi_{t+}$ and $\psi_0$ are deterministic functions of $x_{t+}$ and $x_0$ respectively, so they are conditionally independent from the other random variables.
    In \emph{(c)} we used \cref{asm:symmetrized} after solving for $p(x_{t+} \mid x_0) = p(x_{t+}) e^{-\frac{1}{2}\|A\psi(x_0) - \psi(x)\|_2^2}$.
    In \emph{(d)} we noted that when the integrand is nonzero, it takes on a constant value of $e^{-\frac{1}{2} \|A \psi_0 - \psi_{t+}\|_2^2}$, so we can move that constant outside the integral.
    In \emph{(e)} we used the definition of the marginal representation distribution (\refas{Eq.}{eq:marginal}).
    In \emph{(f)} we used \cref{asm:marginal} to write the marginal distributions $p(\psi_{t+})$ and $p(\psi_0)$ as Gaussian distributions. We removed the normalizing constants, which are independent of $\psi_{t+}$.
    In \emph{(g)} we completed the square and then recognized the expression as the density of a multivariate Gaussian distribution.
\end{proof} %

\subsection{Proof of \cref{thm:planning}: Waypoint Distribution}
\label{appendix:planning}

\restatetheorem{thm:planning}

\begin{proof}
    \begin{align*}
        p(\psi_w \mid \psi_0, \psi_{t+})
         & \stackrel{(a)}{=}  \frac{p(\psi_{t+} \mid \psi_w) p(\psi_w \mid \psi_0)}{\cancel{p(\psi_{t+} \mid \psi_0)}}                                                            \\
         & \stackrel{(b)}{\propto} e^{-\frac{1 + \frac{1}{c}}{2}\|\frac{c}{c+1} A \psi_w - \psi_{t+}\|_2^2} e^{-\frac{1 + \frac{1}{c}}{2}\|\frac{c}{c+1} A \psi_0 - \psi_w\|_2^2} \\
         & \stackrel{(c)}{\propto}  e^{-\frac{1}{2} (\psi_w - \mu)^T \Sigma^{-1} (\psi_w - \mu)}  = \gN(\psi_w; \mu, \Sigma)                                                      \\
    \end{align*}
    where $\Sigma^{-1} = \frac{c}{c+1}A^T A + \frac{c+1}{c}I$ and $\mu = \Sigma (A^T \psi_{t+} + A \psi_0)$.
\end{proof}
In line \emph{(a)} we used the definition of the conditional distribution and then simplified the numerator using the Markov property.
Line \emph{(b)} uses the \cref{lemma:prediction}.
Line \emph{(c)} completes the square, the details of which are below:
\begin{align*}
     & \hspace{-2em}\frac{1}{2} \cdot \frac{c+1}{c}\biggl(\Bigl\|\frac{c}{c+1}A \psi_w -
    \psi_{t+}\Bigr\|_2^2 + \Bigl\| \frac{c}{c+1}A \psi_0 - \psi_w\Bigr\|_2^2 \biggr)                                                                                      \\
     & = \frac{1}{2} \cdot \frac{c+1}{c}\biggl(\psi_w^T \Bigl(\frac{c}{c+1}A\Bigr)^T
    \Bigl(\frac{c}{c+1}A\Bigr) \psi_w - 2 \psi_{t+}^T \Bigl(\frac{c}{c+1}A\Bigr) \psi_w +
    \cancel{\psi_{t+}^T \psi_{t+}} .                                                                                                                                      \\
     & \qquad  + \cancel{\psi_0^T \Bigl(\frac{c}{c+1}A\Bigr)^T \Bigl(\frac{c}{c+1} A\Bigr) \psi_0} - 2\psi_0^T
    \Bigl(\frac{c}{c+1} A\Bigr)^T \psi_w + \psi_w^T \psi_w \biggr)                                                                                                        \\
     & \stackrel{\text{const.}}{=} \frac{1}{2} \cdot \frac{c+1}{c} \Biggl( \psi_w^T \biggl(\Bigl(\frac{c}{c+1}\Bigr)^2 A^TA +
    I\biggr)\psi_w - 2\cdot \frac{c}{c+1}\bigl(A^T \psi_{t+} + A\psi_0\bigr)^T \psi_w \Biggr)                                                                                  \\
     & = \frac{1}{2} \psi_w^T \biggl(\hspace*{3pt}\underbrace{\hspace*{-2pt}\frac{c}{c+1}A^TA + \frac{c+1}{c}I \hspace*{-2pt}}_{\Sigma^{-1}}\hspace*{3pt}\biggr) \psi_w -
    \bigl(A^T \psi_{t+} + A\psi_0\bigr)^T \psi_w                                                                                                                          \\
     & \stackrel{\text{const.}}{=} (\psi_w - \mu)^T \Sigma^{-1} (\psi_w - \mu),
\end{align*}
where $\Sigma^{-1} = \frac{c}{c+1}A^TA + \frac{c+1}{c}I$ and $\mu = \Sigma(A^T \psi_{t+} + A \psi_0)$.
Above, we have used $\stackrel{\text{const.}}{=}$ to denote equality up to an additive constant that is independent of $\psi_w$.

\subsection{Proof of \cref{thm:multistep}: Planning over Many Intermediate States}

\restatetheorem{thm:multistep}

\begin{proof}
    We start by recalling that the waypoints form a Markov chain, so we can express their joint density as a product of conditional densities:
    \begin{equation*}
        p(\psi_{1:n}) = p(\psi_0) p(\psi_1 \mid \psi_0) p(\psi_2 \mid \psi_1) \cdots p(\psi_n \mid \psi_{n-1}).
    \end{equation*}
    The aim of this lemma is to express the joint distribution over multiple waypoints, given an initial and final state representation:
    \begin{align*}
        p(\psi_{1:n} \mid \psi_0, \psi_{t+})
        &\stackrel{(a)}{=} \frac{p(\psi_{1:n} \mid \psi_0) p(\psi_{t+} \mid \psi_n)}{\cancel{p(\psi_{t+}
            \mid \psi_0)}} \\
        &\stackrel{(b)}{\propto} p(\psi_1 \mid \psi_0) p(\psi_2 \mid \psi_1) \cdots p(\psi_{t+} \mid
            \psi_n) \\
        &\stackrel{(c)}{\propto} \exp\Bigl(-\tfrac{1}{2}\tfrac{c+1}{c} \|\tfrac{c}{c+1} A \psi_0 -
            \psi_1\|_2^2 -\tfrac{1}{2}\tfrac{c+1}{c}
            \|\tfrac{c}{c+1} A \psi_1 \nonumber \\
        &\qquad\qquad - \psi_2\|_2^2 - \cdots -\tfrac{1}{2}\tfrac{c+1}{c} \|\tfrac{c}{c+1} A \psi_n -
            \psi_{t+}\|_2^2 \Bigr) \\
        &\stackrel{(d)}{=} \exp\biggl(
            - \cancel{\tfrac{1}{2}\tfrac{c}{c+1} \psi_0^T A^T A \psi_0} + \psi_0^T A^T \psi_1 - \tfrac{1}{2}
            \tfrac{c+1}{c} \psi_1^T \psi_1 \\
        &\qquad\qquad -\tfrac{1}{2}\tfrac{c}{c+1} \psi_1^T A^T A \psi_1 + \psi_1^T A^T \psi_2 -
            \tfrac{1}{2} \tfrac{c+1}{c} \psi_2^T \psi_2 \\
        &\qquad\qquad -\tfrac{1}{2}\tfrac{c}{c+1} \psi_2^T A^T A \psi_2 + \psi_2^T A^T \psi_3 -
            \tfrac{1}{2} \tfrac{c+1}{c} \psi_3^T \psi_3 \\
        &\kern 4cm \vdots \\
        &\qquad\qquad -\tfrac{1}{2}\tfrac{c}{c+1} \psi_n^T A^T A \psi_n + \psi_n^T A^T \psi_{t+} -
            \cancel{\tfrac{1}{2} \tfrac{c+1}{c} \psi_{t+}^T
            \psi_{t+}} \biggr) \\
        &= \exp\bigl( -\tfrac{1}{2}\psi_{1:n}^T \Sigma^{-1} \psi_{1:n} + \eta^T \psi_{1:n} \bigr),
    \end{align*}
    where
    \begin{equation*}
        \Sigma^{-1} = \begin{pmatrix}
                            \frac{c}{c+1}A^TA + \frac{c+1}{c}I & -A^T                               &                                    &                      \\
                            -A                                 & \frac{c}{c+1}A^TA + \frac{c+1}{c}I & -A^T                               &                      \\
                                                               & -A                                 & \frac{c}{c+1}A^TA + \frac{c+1}{c}I &                      \\[-2px]
                                                               &                                    &                                    & \ddots \hspace*{3px} \\[5px]
                        \end{pmatrix} \text{ and }
        \eta = \begin{pmatrix}
                   A \psi_0 \\
                   0        \\[-1ex]
                   \vdots   \\[-2ex]
                   \\
                   0        \\
                   A^T \psi_{t+}
               \end{pmatrix}.
    \end{equation*}

    In \emph{(a)} we applied Bayes' rule and removed the denominator because it is a constant with respect to $\psi_{1:n}$.
    In \emph{(b)} we applied the Markov assumption.
    In \emph{(c)} we used \cref{lemma:prediction} to express the conditional probabilities as Gaussians, ignoring the proportionality constants (which are independent of $\psi$.
    In \emph{(d)} we simplified the exponents, removing terms that do not depend on $\psi_{1:n}$.
\end{proof}

\subsection{Formalizing \cref{asm:symmetrized}}
\label{appendix:symmetrized}

\Cref{asm:symmetrized} relates the learned contrastive critic to a log-likelihood ratio between the positive and negative data distribution.

\restatetheorem{asm:symmetrized}

We can justify this assumption by analyzing the general solution to the symmetrized version of the \citet{oord2018representation} infoNCE objective, which we do in \cref{lemma:full_infonce}.
Applying this lemma to our representation learning objective (\ref{eq:info-nce}) for sufficiently large batch size $B$ then yields \cref{eq:bayes-opt}, with the function approximator $\|\phi(x)-\psi(x^{+})\|^2 \approx f(x,x^{+})$.

\begin{lemma} \label{lemma:full_infonce}
    The solution to the optimization problem
    \begin{equation}
        \max_{f(x,x^{+})}\lim_{B \to \infty} \E_{\{(x_{i}, x_{i}^+)\}_{i=1}^{B} \sim p(x,x^+)} \biggl[\tfrac1B\sum_{i=1}^{B}\log \tfrac{e^{
        f(x_{i},x_{i}^+)}}{\sum_{j\neq i} e^{ f(x_{i},x_j^+)}} +\log \tfrac{e^{ f(x_{i},x_{i}^+)}}{\sum_{j\neq i} e^{ f(x_{j},x_i^+)}} \biggr]
        \label{eq:full_infonce}
    \end{equation}
    satisfies
    \begin{equation}
        f(x, x^+) = \log \biggl(\frac{p( x^+ \mid x)}{p(x^{+}) C}\biggr) \label{eq:full_infonce_soln}
    \end{equation}
    for some $C$.
\end{lemma}

\begin{proof}[Proof of \cref{lemma:full_infonce}]

    We first break down the LHS and RHS of \cref{eq:info-nce}:
\begin{align*}
        \max_{f} \lim_{B \to \infty} &\E_{\{(x_{i}, x_{i}^+)\}_{i=1}^{B} \sim p(x,x^+)}
            \Biggl[\frac1B\sum_{i=1}^{B}\underbrace{\log \frac{e^{
            f(x_{i},x_{i}^+)}}{\sum_{j\neq i} e^{ f(x_{i},x_j^+)}}}_{\mathclap{\mathcal{J}_1}}
            +\underbrace{\log \frac{e^{
            f(x_{i},x_{i}^+)}}{\sum_{j\neq i} e^{ f(x_{j},x_i^+)}}}_{\mathclap{\mathcal{J}_2}} \Biggr] \\
        \mathcal{J}_{1}(f)
        &= \lim_{B \to \infty}\E_{\{(x_{i}, x_{i}^+)\}_{i=1}^{B} \sim p(x,x^+)}
            \Biggl[\frac1B\sum_{i=1}^{B}\log \frac{e^{
            f(x_{i},x_{i}^+)}}{\sum_{j\neq i} e^{ f(x_{i},x_j^+)}} \Biggr] \\
        \mathcal{J}_{2}(f)
        &= \lim_{B \to \infty}\E_{\{(x_{i}, x_{i}^+)\}_{i=1}^{B} \sim p(x,x^+)}
            \Biggl[\frac1B\sum_{i=1}^{B}\log \frac{e^{
            f(x_{i},x_{i}^+)}}{\sum_{j\neq i} e^{ f(x_{j},x_i^+)}} \Biggr]
    \end{align*}

    We now use the following result from \citet{Ma2018NoiseCE}:
    \begin{lemma}
        \label{lemma:infonce_soln}
        The optimal solutions $f_1$ and $f_2$ for $\mathcal{J}_1$ and $\mathcal{J}_2$ satisfy
        \begin{align}
            f_1(x, x^+) & = \log p(x \mid x^+) - \log c_1(x) \label{eq:opt_f1}      \\
            f_2(x, x^+) & = \log p(x^+ \mid x ) - \log c_2(x^{+}) \label{eq:opt_f2}
        \end{align}
        for arbitrary $c_1(x), c_2(x^{+})$.
    \end{lemma}

    For any $C$, when $c_1(x) = C p(x)$ and $c_2(x^{+}) = Cp(x^{+})$,
    \begin{align}
        f_1(x, x^+) & = \log \biggl(\frac{p( x \mid x^+)}{p(x) C}\biggr)
        = \log \biggl(\frac{p( x^{+} \mid x)}{p(x^{+}) C}\biggr)
        = f_2(x, x^+). \label{eq:joint_optimal}
    \end{align}

    It follows that \cref{eq:joint_optimal} maximizes both $\mathcal{J}_1$ and $\mathcal{J}_2$, and is precisely the optimal solution \cref{eq:full_infonce_soln} for \cref{eq:full_infonce}.
\end{proof}

\paragraph{What does $C$ represent?} From \cref{eq:full_infonce_soln}, we can connect $C$ to the mutual information $I(x,x^{+})$:
\begin{equation}
    C = \frac{\mathbb{E}_{(x,x^{+}) \sim p(x,x^{+})} \bigl[ f(x,x^{+}) \bigr]}{{I(x,x^{+})}}.
    \label{eq:mi_connection}
\end{equation}

\begin{proof}[Proof of \cref{lemma:infonce_soln}]

    We can first consider $\mathcal{J}_1$ without loss of generality. Denoting
    \[
        g(x,x^{+})=e^{f(x,x^{+})},
    \]
    we take the functional derivative:
    \begin{align*}
        \delta \mathcal{J}_1(\log g)
         & =\lim_{B \to \infty}\delta \E_{\{(x_{i}, x_{i}^+)\}_{i=1}^{B} \sim p(x,x^+)} \biggl[\tfrac1B\sum_{i=1}^{B}\log\frac{
        g(x_{i},x_{i}^+)}{\sum_{j\neq i} g(x_{i},x_j^+)} \biggr] \nonumber                                                                           \\
         & =\lim_{B \to \infty}\E_{\{(x_{i}, x_{i}^+)\}_{i=1}^{B} \sim p(x,x^+)}\biggl[\tfrac1B\sum_{i=1}^{B}\tfrac{ (\sum_{j\neq i} g(x_{i},x_j^+))
        \delta g(x_{i},x_{i}^+) - g(x_{i},x_{i}^+) \delta (\sum_{j\neq i} g(x_{i},x_j^+))}{{ g(x_{i},x_{i}^+)}(\sum_{j\neq i}
        g(x_{i},x_j^+))}\biggr] \nonumber                                                                                                            \\
         & =\lim_{B \to \infty}\E_{\{(x_{i}, x_{i}^+)\}_{i=1}^{B}\sim p(x,x^+)}\biggl[\tfrac1B\sum_{i=1}^{B}\tfrac{ \delta g(x_{i},x_{i}^+) }{{
        g(x_{i},x_{i}^+)}} -\tfrac{ \delta (\sum_{j\neq i} g(x_{i},x_j^+))}{\sum_{j\neq i} g(x_{i},x_j^+)}\biggr] \nonumber                          \\
         & =\lim_{B \to \infty}\E_{\{(x_{i}, x_{i}^+)\}_{i=1}^{B}\sim p(x_{i},x^+)}\biggl[\tfrac{1}B\sum_{i=1}^{B} \int \Bigl(\bigl(\tfrac{ \delta
        g(x_{i},x^+) }{{ g(x_{i},x^+)}}\bigr) p(x^{+} \mid x_{i}) \nonumber                                                                          \\*
         & \hspace*{4.7cm} - \sum_{k\neq i} \bigl(\tfrac{ \delta g(x_{i},x^{+})}{ g(x_{i},x^{+}) - g(x_{i},x^{+}_{k}) + \sum_{j\neq i}
        g(x_{i},x_j^+)}\bigr) p(x^{+}) \Bigr)\dd x^{+}\biggr]                                                                                        \\
         & =\lim_{B \to \infty}\E_{\{(x_{i}, x_{i}^+)\}_{i=1}^{B}\sim p(x,x^+)}\biggl[\tfrac{1}B\sum_{i=1}^{B} \nonumber \int \delta g(x_{i},x^+)
        \Bigl(\tfrac{p(x^{+} \mid x_{i})}{{ g(x_{i},x^+)}}                                                                                           \\*
         & \hspace*{4.7cm} - \underbrace{\E_{\{(x_{i}, x_{i}^+)\}_{i=1}^{B}}\Bigl[\tfrac{1}{ \sum_{j\neq i}
        g(x_{i},x_j^+)}\Bigr]}_{\mathclap{\text{as $B\to \infty$}}} p(x^{+})\Bigr) \dd x^{+}\biggr]                                                  \\
         & =\lim_{B \to \infty}\E_{\{(x_{i}, x_{i}^+)\}_{i=1}^{B}\sim p(x,x^+)}\biggl[\tfrac{1}B\sum_{i=1}^{B} \nonumber \int \delta g(x_{i},x^+)
        \Bigl(\tfrac{p(x^{+} \mid x_{i})}{{ g(x_{i},x^+)}}                                                                                           \\*
         & \hspace*{4.7cm} - \underbrace{\E_{\{(x_{i}, x_{i}^+)\}_{i=1}^{B}}\Bigl[\tfrac{1}{ \sum_{j\neq i}
        g(x_{i},x_j^+)}\Bigr]}_{\mathclap{\text{$\triangleq\!k(x_i)$ indep. of $x^+$}}} p(x^{+})\Bigr) \dd x^{+}\biggr]                              \\
         & =\lim_{B \to \infty}\E_{\{(x_{i}, x_{i}^+)\}_{i=1}^{B}\sim p(x,x^+)}\biggl[\tfrac{1}B\sum_{i=1}^{B} \int \delta g(x_{i},x^+)
        \Bigl(\tfrac{p(x^{+} \mid x_{i})}{{ g(x_{i},x^+)}} - k(x_{i}) p(x^{+})\Bigr) \dd x^{+}\biggr]                                                \\
         & = \int \delta g(x,x^+) \bigl(\tfrac{p(x^{+} \mid x)}{{ g(x,x^+)}} - k(x) p(x^{+})\bigr) \dd x^{+} .
    \end{align*}

    This is zero when
    \[
        g(x,x^{+}) = \frac{p(x \mid x^{+})}{k(x)p(x)}  ,
    \]
    i.e.,
    \[
        f(x,x^{+}) = \log p(x \mid x^{+}) - \log \underbrace{c_1(x)}_{\mathclap{k(x)p(x)}}
    \]
    as in \cref{eq:opt_f1}, and \cref{eq:opt_f2} follows similarly, exchanging $x$ and $x^{+}$.
\end{proof}

\section{Additional Experiments}

\begin{figure*}
    \centering%
    \includegraphics[width=0.2\linewidth]{figures/point_plan_08.png}%
    \includegraphics[width=0.2\linewidth]{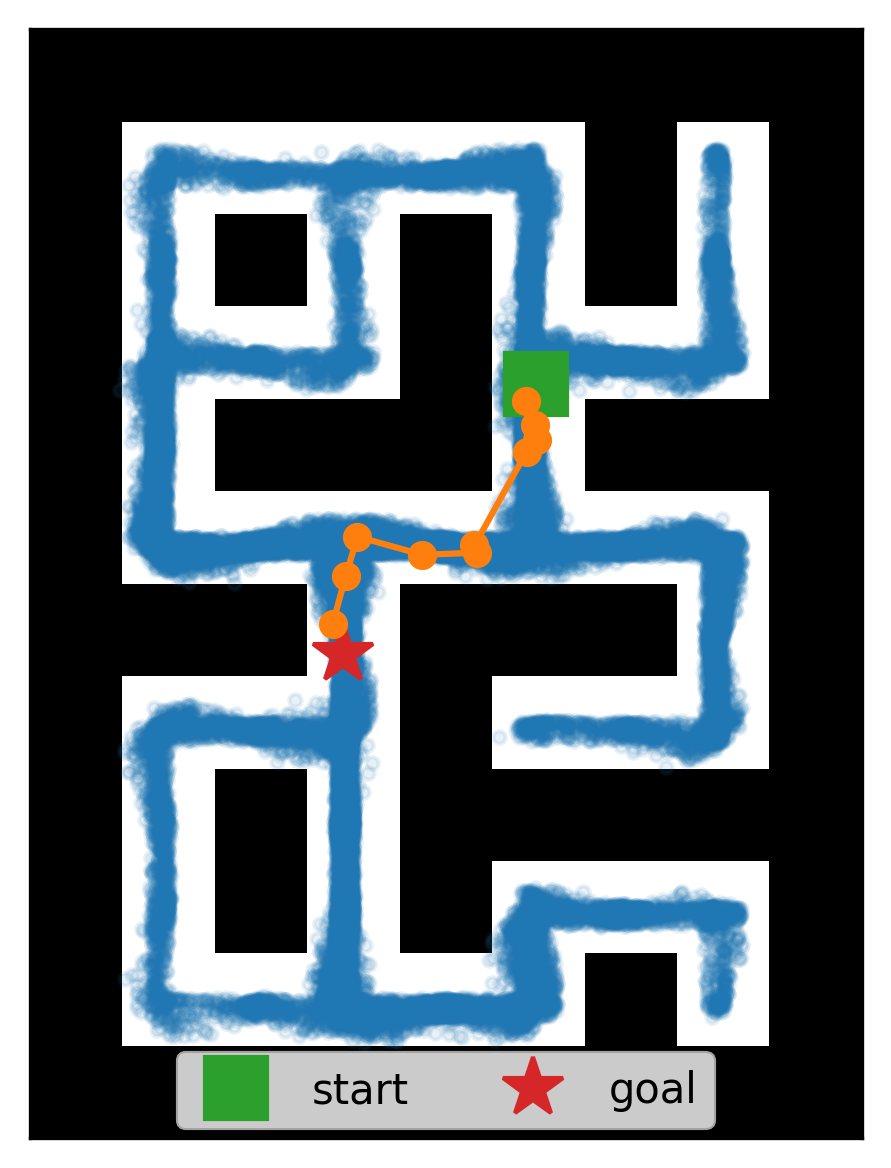}%
    \includegraphics[width=0.2\linewidth]{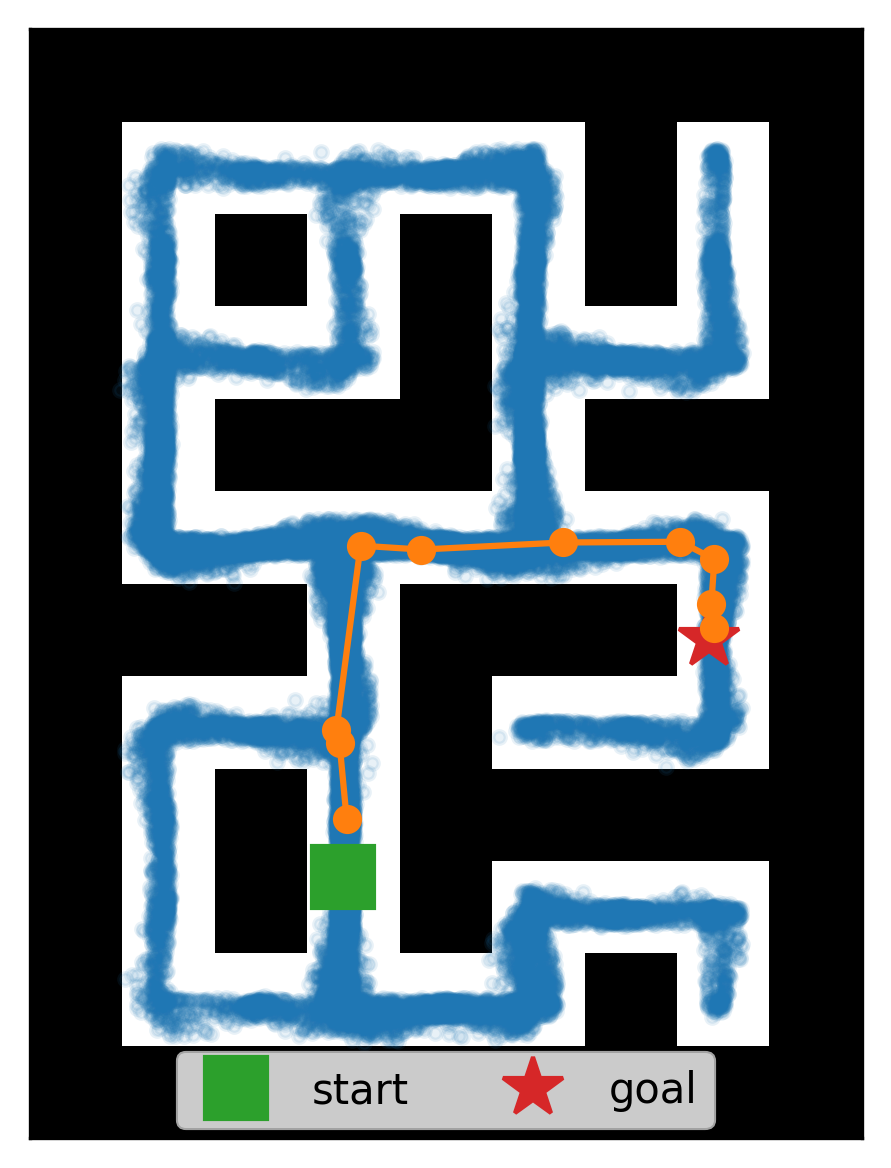}%
    \includegraphics[width=0.2\linewidth]{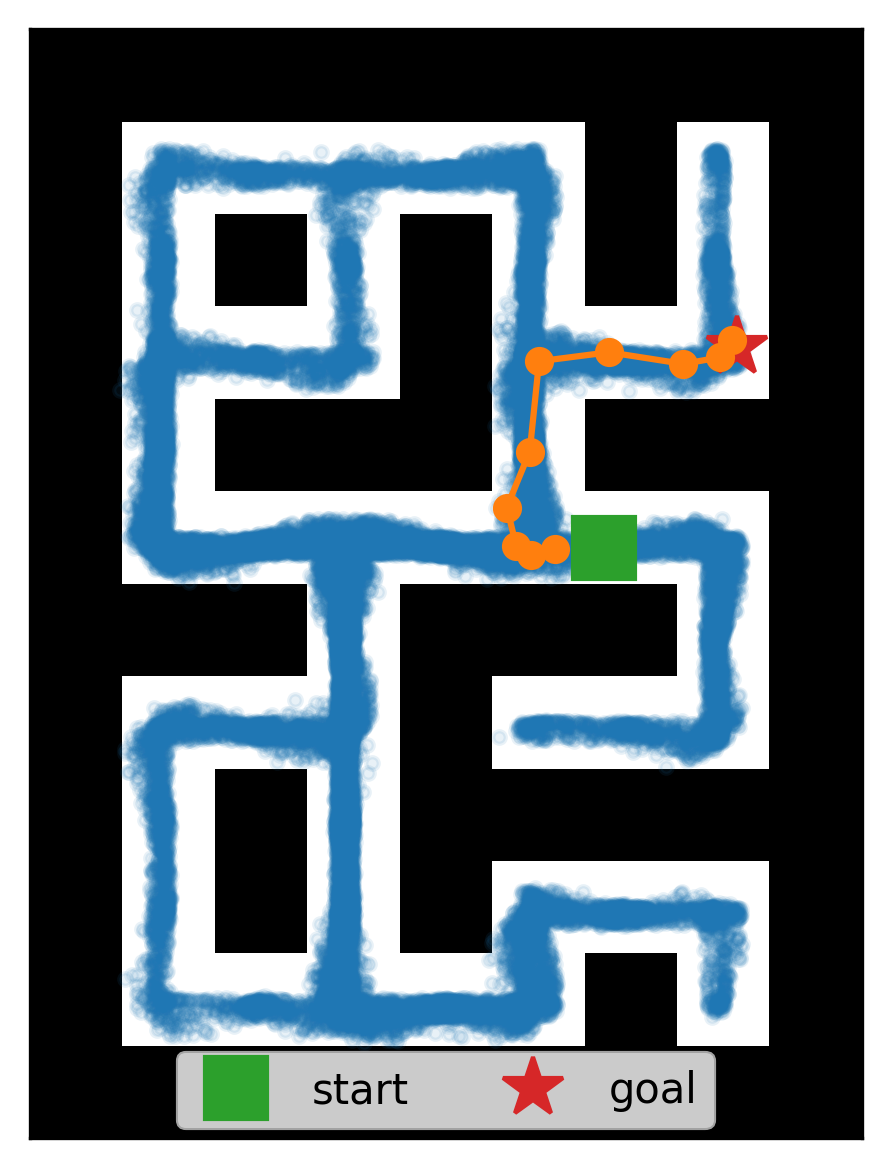}%
    \includegraphics[width=0.2\linewidth]{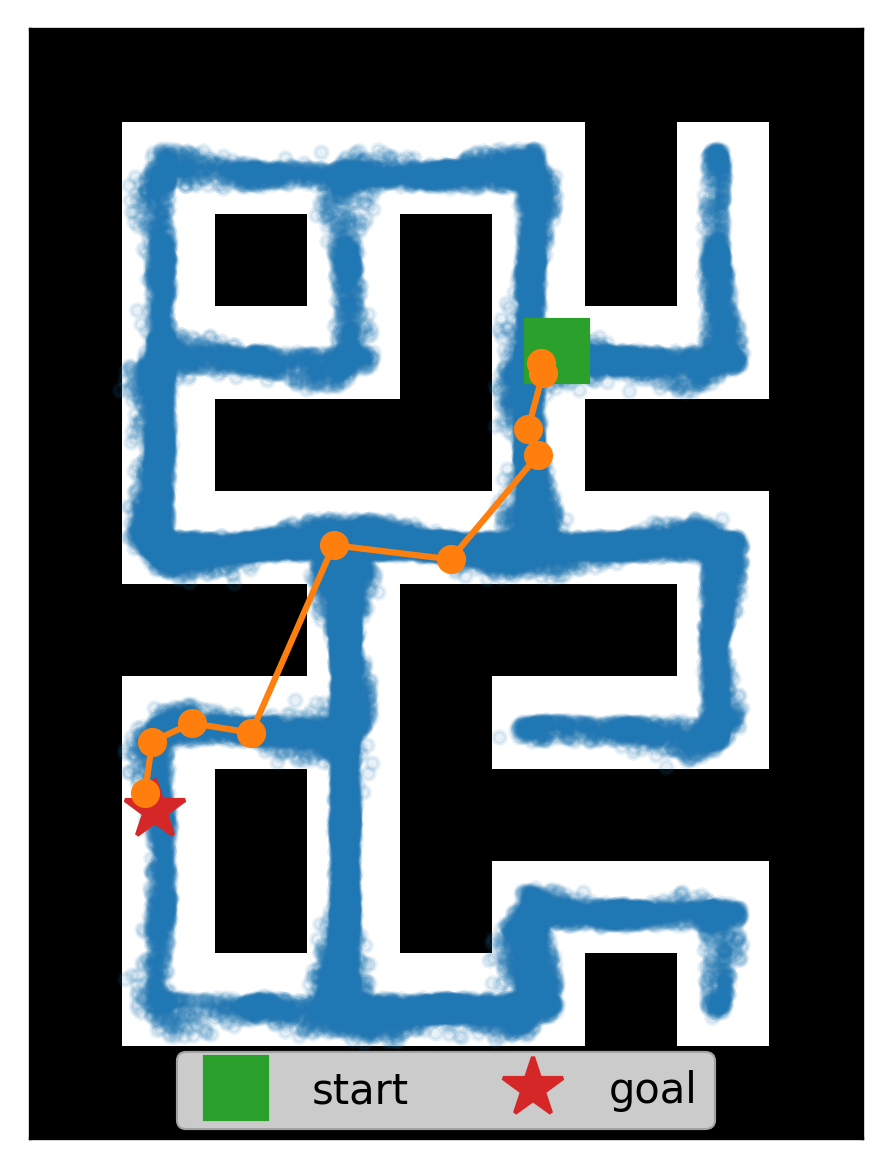}%
    \caption{Our approach enables a goal-conditioned policy to reach farther targets (red) from the start (green) by planning over intermediate waypoints (orange).}%
    \label{fig:maze}
\end{figure*}

\cref{fig:maze} visualizes the inferred waypoints from the task in \cref{fig:maze-results}.
\cref{fig:hammer} visualizes the representations learned on a 46-dimensional robotic hammering task (see \cref{sec:high-dim}).

\begin{figure*}
    \centering
    \begin{subfigure}[c]{0.2\textwidth}
        \centering
        \includegraphics[width=\linewidth]{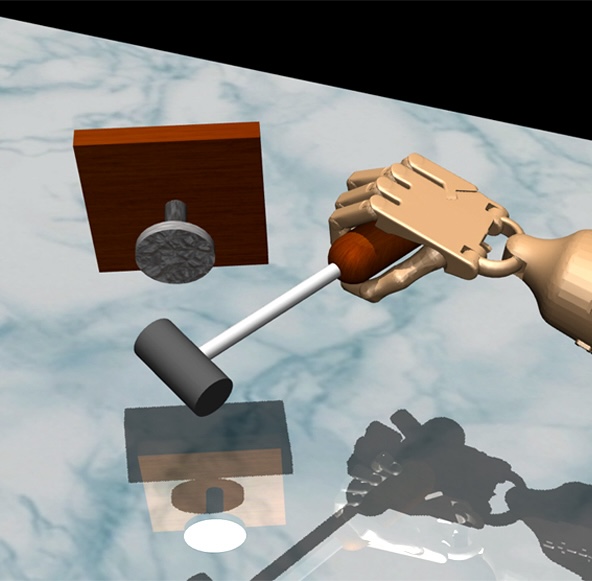}
    \end{subfigure}%
    \kern 0.05\textwidth
    \begin{subfigure}[c]{0.45\textwidth}
        \centering
        \newdimen\w
        \w=.21\linewidth
        \parbox[c]{\w}{\centering\includegraphics[width=\w]{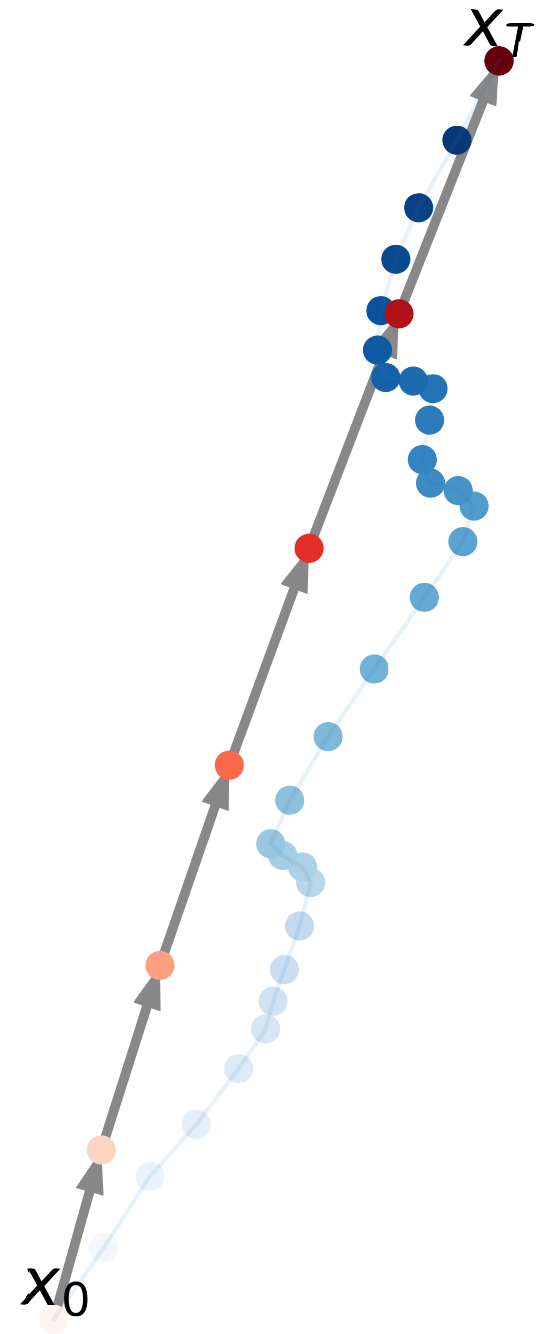}}%
        \w=.47\linewidth%
        \kern 1em
        \parbox[c]{\w}{\centering\includegraphics[width=\w]{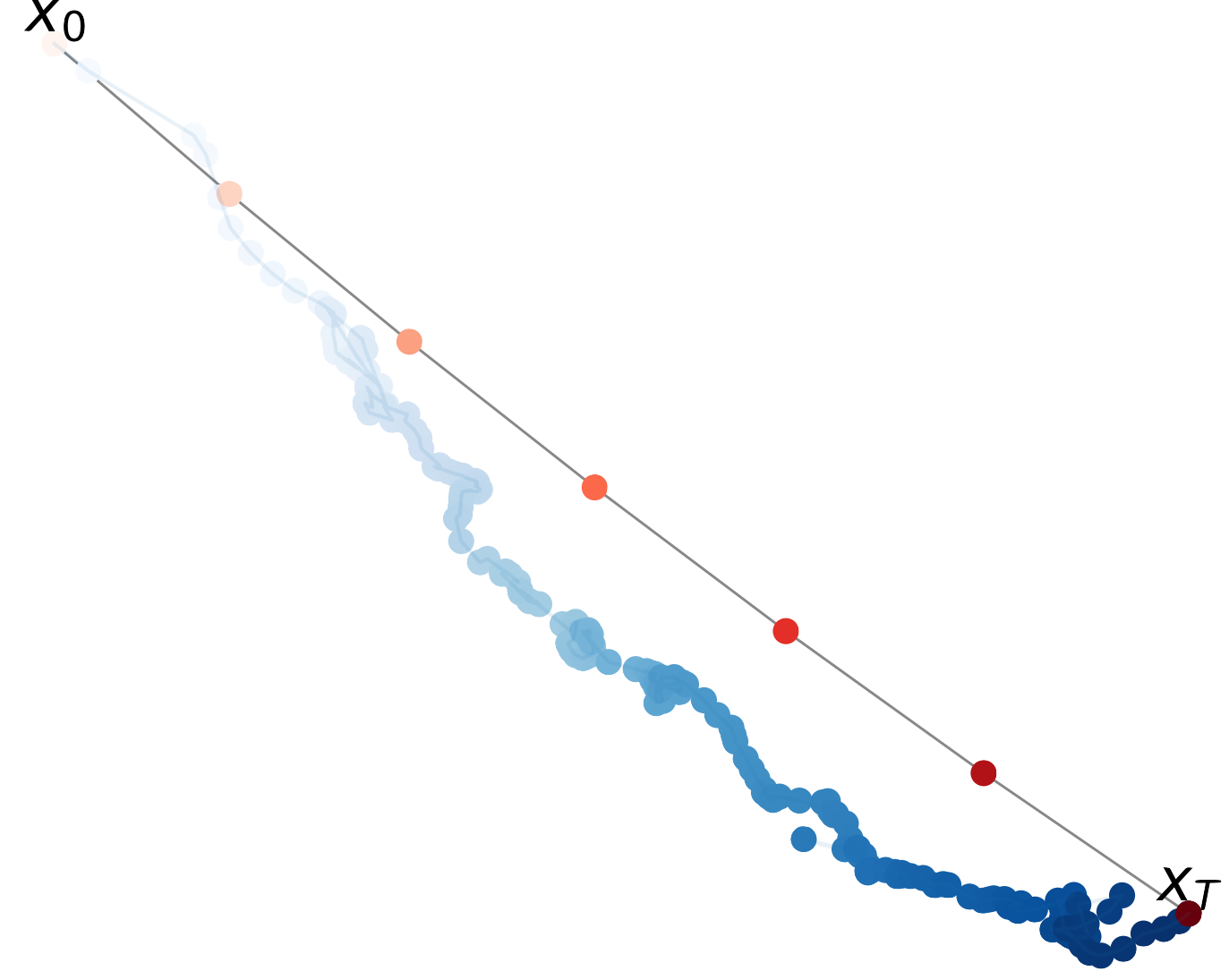}}%
        \w=.34\linewidth%
        \kern -2em%
        \parbox[c]{\w}{\centering\includegraphics[width=\w]{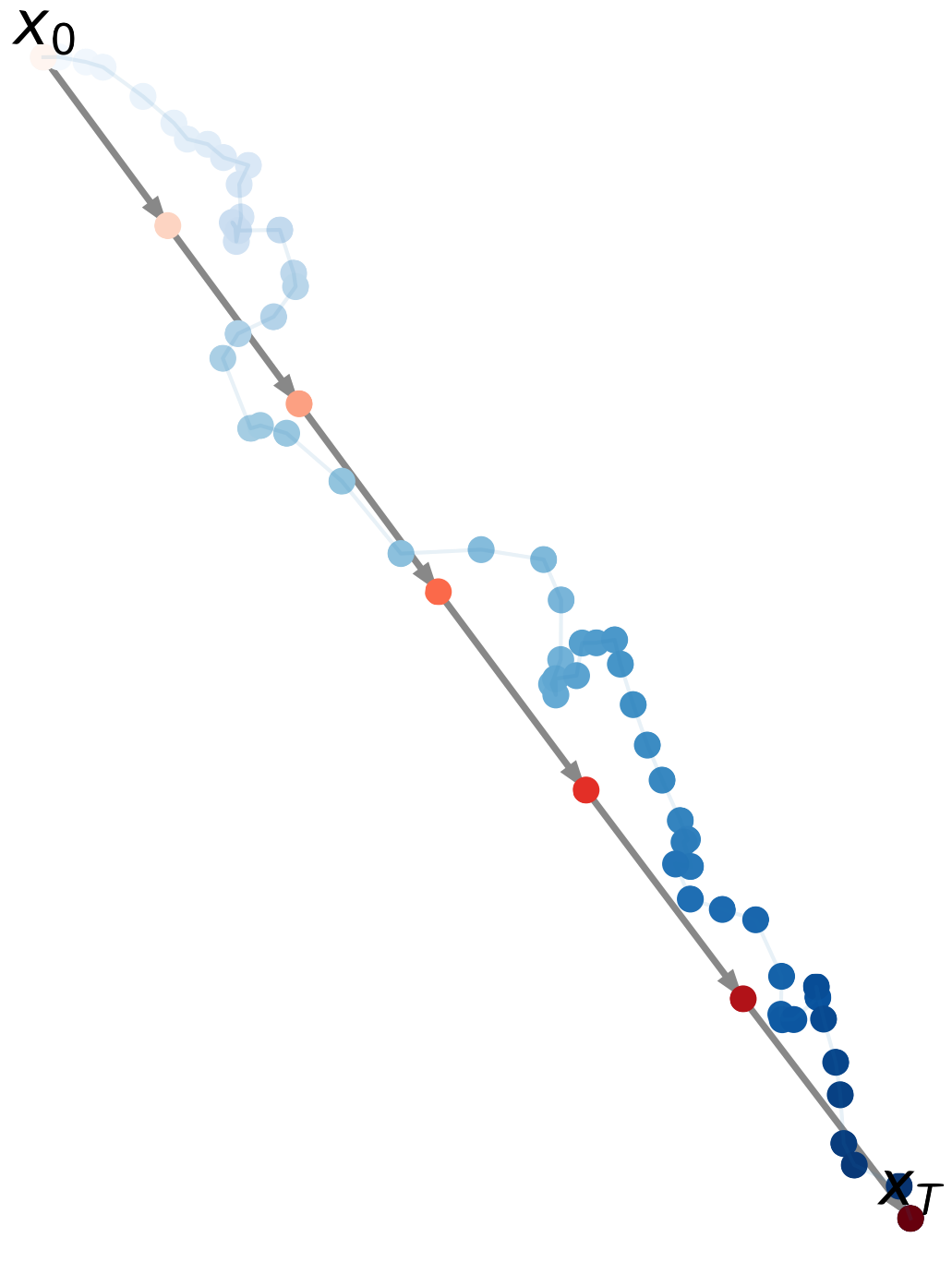}}%
        \caption{Contrastive representations}
    \end{subfigure}%
    \kern 0.05\textwidth
    \begin{subfigure}[c]{0.25\textwidth}
        \centering
        \includegraphics[height=3cm]{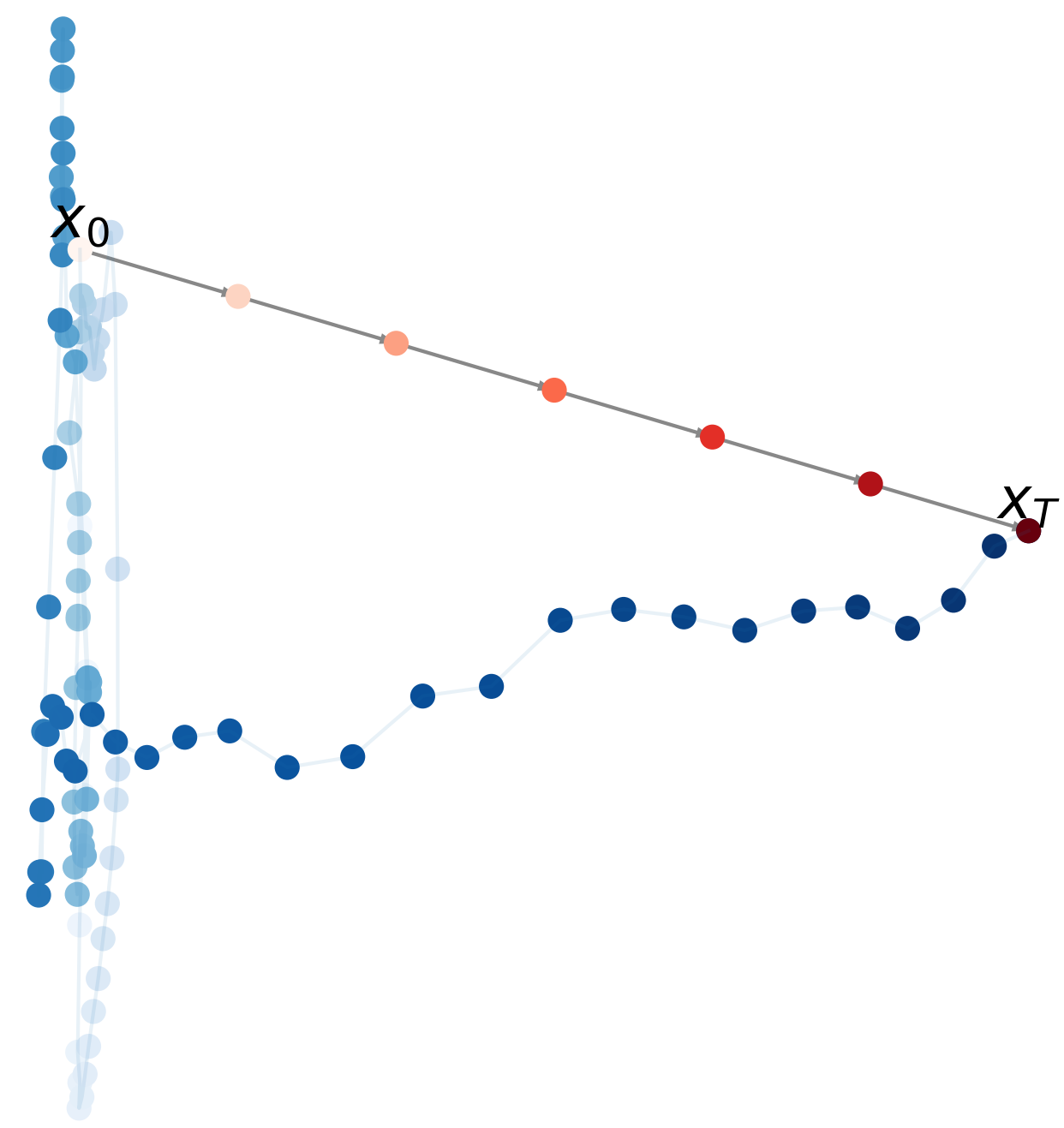}
        \caption{PCA representations}
    \end{subfigure}
    \caption{Planning for 46-dimensional robotic hammering.
        \figleft \, A dataset of trajectories demonstrating a hammer knocking a nail into a board~\citep{fu2020d4rl}. \figcenter \, We visualize the learned representations as blue circles, with the transparency indicating the index of that observation along the trajectory. We also visualize the inferred plan (\cref{sec:planning}) as red circles connected by arrows.
        \figright \, Representations learned by PCA on the same trajectory as \emph{(a, left)}.
    }
    \label{fig:hammer}
\end{figure*}

\subsection{Stock Prediction}
We show results on a stock opening price task in \cref{fig:stocks}.

\begin{figure}[htb]
    \centering
    \includegraphics[width=\linewidth]{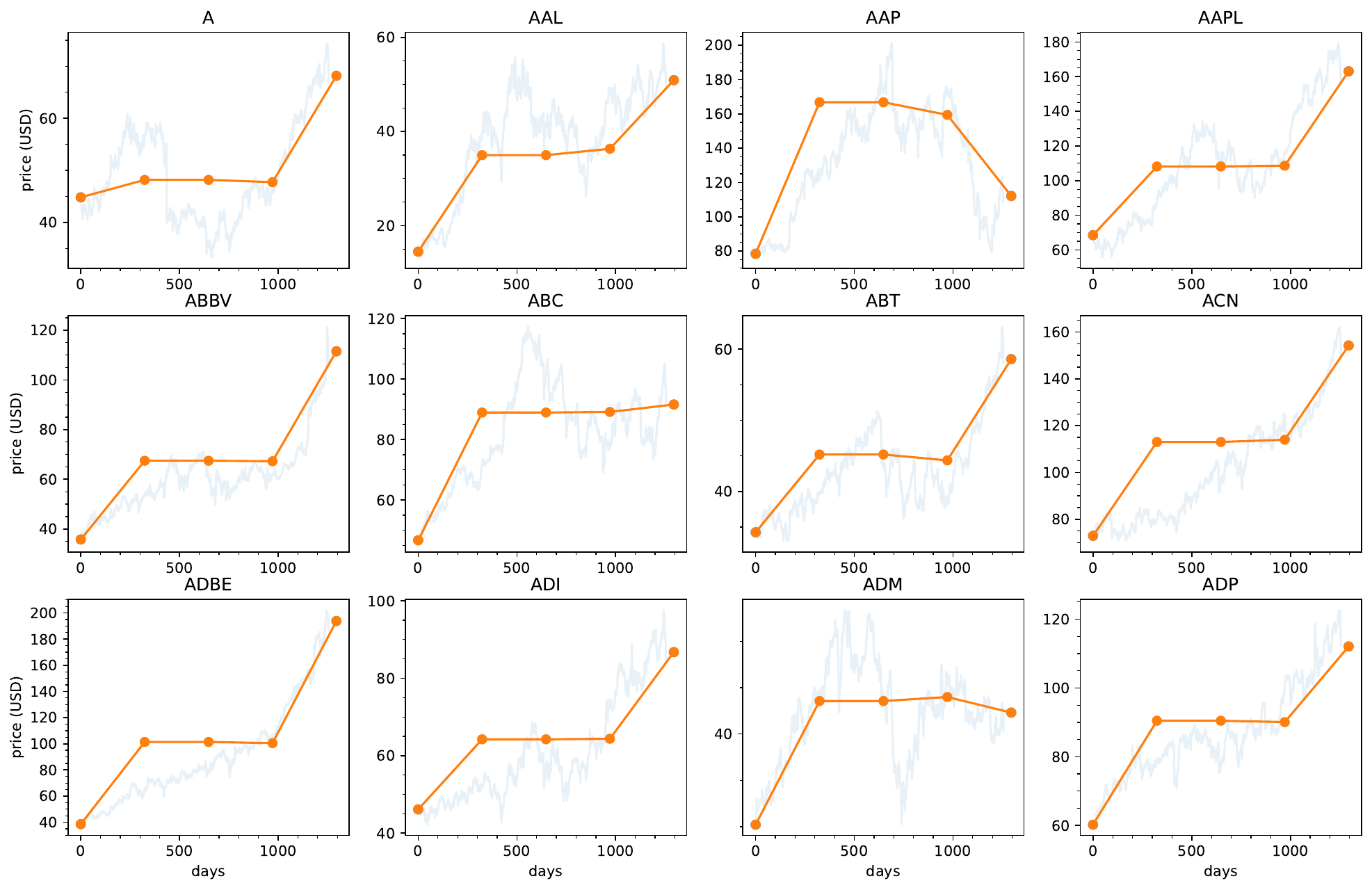}
    \caption{\textbf{Stock Prediction.} We apply temporal contrastive learning to time series data of the stock market.
        Data are the opening prices for the 500 stocks in the S\&P 500, over a four year window.
        We remove 30 stocks that are missing data.
        For evaluation, we choose a 100 day window from a validation set, and use \cref{thm:planning} to perform ``inpainting,'' predicting the intermediate stock prices \emph{jointly} for all stocks ({\color{orange}orange}), given the first and last stock price.
        The true stock prices are shown in {\color{blue}blue}.
        While we do not claim that this is a state-of-the-art model for stock prediction, this experiment demonstrates another potential application of our theoretical results.}
    \label{fig:stocks}
\end{figure}

\end{document}